\documentclass{article}

\usepackage[main,final]{extra/neurips_2026}

\usepackage[utf8]{inputenc} 
\usepackage[T1]{fontenc}    
\usepackage{hyperref}       
\usepackage{url}            
\usepackage{booktabs}       
\usepackage{tabularx}      
\usepackage[table]{xcolor}
\usepackage{multirow}
\usepackage{amsfonts}       
\usepackage{nicefrac}       
\usepackage{microtype}      
\usepackage{xcolor}         
\usepackage{amsmath}
\usepackage{amssymb}
\usepackage{amsthm}
\usepackage{mathtools}
\usepackage{cleveref}
\usepackage{algorithm}
\usepackage{algorithmic}
\usepackage{lipsum}
\usepackage{setspace}
\usepackage{thm-restate}
\usepackage{enumitem}
\usepackage{subcaption}
\crefname{section}{Sec.}{Sec.}
\crefname{thm}{Thm.}{Theorem}
\crefname{corr}{Cor.}{Corollary}
\crefname{appendix}{App.}{Appendices}
\crefname{algorithm}{Alg.}{Algorithms}
\crefname{equation}{Eq.}{Eqs.}
\crefname{figure}{Fig.}{Figs.}
\crefname{prop}{Prop.}{Props.}
\creflabelformat{equation}{#2\textup{#1}#3}

\definecolor{TolMutedBlue}{HTML}{332288}
\definecolor{TolMutedGreen}{HTML}{117733}
\definecolor{TolMutedPurple}{HTML}{AA4499}
\definecolor{TolMutedRed}{HTML}{BB3322}

\hypersetup{
    colorlinks=true,
    citecolor=TolMutedBlue,%
    linkcolor=TolMutedGreen,%
    urlcolor=TolMutedPurple,%
}

\newcommand{\loss}{\ell}

\newcommand{\vparam}{\boldsymbol{\theta}}

\newcommand{\vParam}{\vTheta}

\newcommand\cut[1]{}

\newcommand{\squishlist}{
   \begin{list}{$\bullet$}
    { \setlength{\itemsep}{0pt}      \setlength{\parsep}{3pt}
      \setlength{\topsep}{3pt}       \setlength{\partopsep}{0pt}
      \setlength{\leftmargin}{1.5em} \setlength{\labelwidth}{1em}
      \setlength{\labelsep}{0.5em} } }

\newcommand{\squishlisttwo}{
   \begin{list}{$\bullet$}
    { \setlength{\itemsep}{0pt}    \setlength{\parsep}{0pt}
      \setlength{\topsep}{0pt}     \setlength{\partopsep}{0pt}
      \setlength{\leftmargin}{2em} \setlength{\labelwidth}{1.5em}
      \setlength{\labelsep}{0.5em} } }

\newcommand{\squishend}{
    \end{list}  }

{}
{}
{}

\newtheorem{lemma}{Lemma}{}

{}

\newcommand{\half}{\mbox{$\frac{1}{2}$}}

\newcommand{\rnd}[1]{\left(#1\right)}
\newcommand{\sqr}[1]{\left[#1\right]}

\newcommand{\gauss}{\mbox{${\cal N}$}}

\newcommand{\myvec}[1]{\mbox{$\mathbf{#1}$}}
\newcommand{\myvecsym}[1]{\mbox{$\boldsymbol{#1}$}}

\newcommand{\vone}{\mbox{$\myvecsym{1}$}}

\newcommand{\vepsilon}{\boldsymbol{\epsilon}}

\newcommand{\vmu}{\mbox{$\myvecsym{\mu}$}}

\newcommand{\vTheta}{\mbox{$\myvecsym{\Theta}$}}
\newcommand{\vsigma}{\mbox{$\myvecsym{\sigma}$}}
\newcommand{\vSigma}{\mbox{$\myvecsym{\Sigma}$}}

\newcommand{\va}{\mbox{$\myvec{a}$}}
\newcommand{\vb}{\mbox{$\myvec{b}$}}

\newcommand{\vg}{\mbox{$\myvec{g}$}}
\newcommand{\vh}{\mbox{$\myvec{h}$}}

\newcommand{\vm}{\mbox{$\myvec{m}$}}

\newcommand{\vs}{\mbox{$\myvec{s}$}}

\newcommand{\vw}{\mbox{$\myvec{w}$}}

\newcommand{\vx}{\mbox{$\myvec{x}$}}

\newcommand{\vy}{\mbox{$\myvec{y}$}}

\newcommand{\vA}{\mbox{$\myvec{A}$}}
\newcommand{\vB}{\mbox{$\myvec{B}$}}

\newcommand{\vD}{\mbox{$\myvec{D}$}}
\newcommand{\vE}{\mbox{$\myvec{E}$}}
\newcommand{\vF}{\mbox{$\myvec{F}$}}
\newcommand{\vG}{\mbox{$\myvec{G}$}}
\newcommand{\vH}{\mbox{$\myvec{H}$}}
\newcommand{\vI}{\mbox{$\myvec{I}$}}

\newcommand{\vL}{\mbox{$\myvec{L}$}}
\newcommand{\vM}{\mbox{$\myvec{M}$}}

\newcommand{\vP}{\mbox{$\myvec{P}$}}
\newcommand{\vQ}{\mbox{$\myvec{Q}$}}
\newcommand{\vR}{\mbox{$\myvec{R}$}}

\newcommand{\vT}{\mbox{$\myvec{T}$}}
\newcommand{\vU}{\mbox{$\myvec{U}$}}
\newcommand{\vV}{\mbox{$\myvec{V}$}}

\newcommand{\vX}{\mbox{$\myvec{X}$}}

\newcommand{\be}{\begin{equation}}
\newcommand{\ee}{\end{equation}}
\newcommand{\bea}{\begin{eqnarray}}
\newcommand{\eea}{\end{eqnarray}}
\newcommand{\beaa}{\begin{eqnarray*}}
\newcommand{\eeaa}{\end{eqnarray*}}

\newcommand{\bbR}{\mathbb{R}}

\DeclarePairedDelimiterX{\infdivx}[2]{{}}{{}}{%
  \left( #1\,\delimsize\|\,#2\right)%
}

\DeclareMathOperator{\KLop}{KL}
\newcommand{\myKL}{{\KLop}\infdivx}

\newif\ifrevmarks
\ifrevmarks
  \newcommand{\revision}[1]{{\color{blue}#1}}
  
\else
  \newcommand{\revision}[1]{#1}
  
\fi

\usepackage{pifont}

\title{SOAP-Bubbles $\circ$\ding{109} \\ Structured Weight Uncertainty for Neural Networks}

\author{
Adrian R. Minut$^{1}$\thanks{Part of this work was done during an internship at RIKEN AIP.} \quad Nico Carlsen Daheim$^{2}$ \quad Marco Miani$^{3}$\\
\textbf{Mohammad Emtiyaz Khan}$^{4,5}$ \quad \textbf{Wu Lin}$^{6}$ \quad \textbf{Thomas Möllenhoff}$^{5}$\thanks{Correspondence to: {\tt\small thomas.moellenhoff@riken.jp}}\\[0.2cm]
$^1$Sapienza University of Rome, Rome, Italy\\
$^2$Ubiquitous Knowledge Processing Lab (UKP Lab),\\ Department of Computer Science, Technical University of Darmstadt\\
National Research Center for Applied Cybersecurity ATHENE, Germany\\
$^3$Technical University of Denmark, Lyngby, Denmark\\
$^4$TU Darmstadt \& Hessian Center for AI (hessian.AI), Darmstadt, Germany\\
$^5$RIKEN Center for Advanced Intelligence Project, Tokyo, Japan \\
$^6$University of Central Florida, Orlando, United States \\
}

\begin{document}

\maketitle

\begin{abstract}
     Structured weight-uncertainty can improve many aspects of deep learning, but it remains costly to estimate and difficult to implement. Here, we show that these issues can be addressed by adapting the SOAP optimizer. Our key idea is to run IVON, an existing diagonal-covariance variational method, in the eigenspace of SOAP's preconditioner and then use the preconditioner to transform the diagonal estimate into a non-diagonal covariance. The resulting method has costs similar to those of SOAP and requires no drastic changes to training pipelines. We call the posteriors obtained in this way SOAP-Bubbles and our new optimizer Eigenspace-VON (EVON). We show that, for logistic regression, EVON can match the full Gaussian covariance and that, for language model pretraining, it yields significantly better results than existing diagonal-covariance methods. Our work makes it easier to estimate more expressive posterior distributions for deep learning at scale. \\
     Code is available at: \url{https://github.com/team-approx-bayes/evon}.
\end{abstract}

\section{Introduction}
Weight uncertainty promises many improvements over current deep learning recipes.
For example, it can be useful for exploration in reinforcement learning~\citep{sehnke2010parameter}, \revision{to improve model merging~\citep{daheim2024model}, to improve language generation by using posterior sampling~\citep{daheim2025uncertainty, gan2026neural}, and to understand model behavior~\citep{nickl2023memory}.
The list of applications is long~\citep[inter alia]{NgLi18,OsSw19,wang2024blob,cong2025improving, khan2025knowledge,mollenhoff2025federated}, yet weight uncertainty is rarely used in practice.} 

One reason behind this is that weight-uncertainty can be costly and also increase implementation overheads. Cheaper options do exist, but the quality may not always be acceptable.
For example, the IVON optimizer~\citep{shen2024variational} obtains Gaussian posteriors without significant increase in the cost compared to Adam, but it can only estimate a diagonal covariance which restricts the expressivity of the posterior. As a result, the quality of the uncertainty estimate might not always be sufficient.

Going beyond diagonal covariances remains an open challenge.~There is a lot of work on post-hoc Laplace approximations, for example, with Kronecker-factored structures~\citep{ImBa21,daxberger2021laplace,eschenhagen2023kronecker,dhahri2024shaving,hong2025better}. Unfortunately, such methods require an additional pass through the training data, which comes with extra costs and is not always feasible. 
Variational methods with structured covariance~\citep{bae2018eigenvalue,ZhSu18,louizos2016structured,mishkin2018slang,LiKh19b,pmlr-v139-lin21e,fadel2025viking} are also difficult to implement and can increase memory and compute overheads. No study exists to date to scale such methods to language models. There is a need for a method to estimate structured weight-uncertainty that is easy to implement, adds little overhead, and scales to large problems.

\revision{In this paper,} we introduce \textit{SOAP-Bubbles}, a novel family of expressive, non-diagonal posterior distributions. SOAP-Bubbles are transformations of diagonal covariances, where the transformation is a rotation constructed from SOAP's preconditioner~\citep{vyas2024soap}. By considering a diagonal posterior within a rotated subspace, SOAP-Bubbles can capture block-diagonal covariances in the original space without the prohibitive costs usually associated with structured covariance approximations. We illustrate the main idea in~\Cref{fig:fig1} (left), where \revision{the SOAP-Bubble provides a better fit to the exact posterior than an axis-aligned diagonal Gaussian}. In our theoretical contribution, we \revision{also} show that SOAP-Bubbles can recover the optimal full-Gaussian posterior approximation \revision{in some cases}, for example, \revision{for} binary logistic regression.

To \revision{train such} structured posteriors at scale, we propose the Eigenspace Variational Online Newton (EVON) algorithm. EVON entails only a simple modification to the SOAP optimizer~\citep{vyas2024soap}. SOAP (ShampoO with Adam in the Preconditioner's Eigenbasis) constructs a block-diagonal preconditioner from stochastic minibatch gradients and runs Adam within the resulting eigenspace. EVON builds directly on this structure \revision{and operates} in the same eigenspace, \revision{but replaces} Adam with the mean-field variational method IVON~\citep{shen2024variational}. This maps a diagonal variational method in the eigenspace to a \revision{more} expressive posterior in the original weight space, achieving the performance of structured variational inference without \revision{major} changes to standard training pipelines.

\begin{figure}[t!]
  \includegraphics[width=\linewidth]{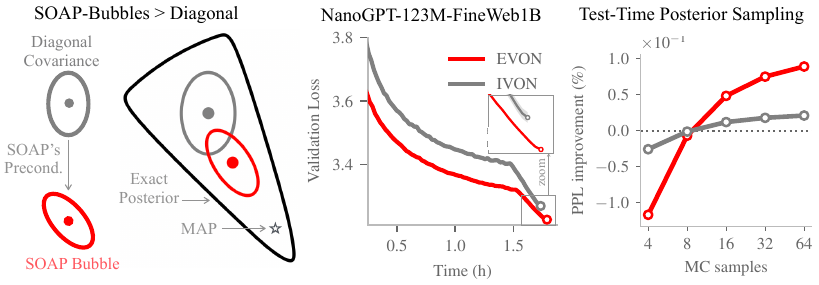}
  \caption{Our key idea is to use SOAP's preconditioner to transform diagonal covariances into non-diagonal ones. We call the resulting posteriors SOAP-Bubbles and \revision{show that} they provide a much better fit to the exact posterior \revision{(left)}. Our new optimizer EVON runs the state-of-the-art variational method IVON in SOAP's eigenspace and significantly improves both the training speed and final loss for language models \revision{(middle)}. SOAP-Bubble posteriors obtained with EVON also work much better for model averaging during test-time than IVON's diagonal posteriors \revision{(right)}. 
  \label{fig:fig1}
}
\end{figure}

In our experiments, we demonstrate that EVON allows SOAP-Bubbles to scale to language model training, yielding consistently better results than IVON under the same computational budget. For example, when training NanoGPT on FineWeb-1B~\citep{fineweb}, EVON significantly improves validation loss, as shown in \Cref{fig:fig1} (middle). Furthermore, ensembling models sampled from SOAP-Bubbles substantially outperforms ensembling from IVON's diagonal posterior, see \Cref{fig:fig1} (right). This confirms that our approach not only improves training optimization dynamics but successfully captures a richer representation of weight uncertainty.

\section{Variational Learning and Structured Weight Uncertainty}
We start with an introduction of variational learning (VL) and then discuss \revision{the} challenges of going beyond diagonal \revision{co}variances in VL. In deep learning we usually minimize a loss function $\ell(\vparam) = \sum_{i=1}^n \ell_i(\vparam) / n$ over $n$ data examples with respect to the parameters $\vparam$. To model weight uncertainty, VL instead searches for a distribution over the parameters and targets a related objective over $q(\vparam)$,  
\begin{equation}
  \min_{q \in \mathcal{Q}} ~ \zeta \, \mathbb{E}_{q(\text{$\vparam$})}[\ell(\vparam)] + \myKL{q(\vparam)}{p(\vparam)},
  \label{eq:varlearn}
\end{equation}
where $\mathcal{Q}$ is a set of candidate posterior distributions. The Kullback-Leibler term regularizes towards a prior $p(\vparam)$, here assumed to be an isotropic Gaussian with zero-mean, and is governed by $\zeta^{-1} > 0$. 
This corresponds to a quadratic regularizer (non-decoupled weight-decay) and in this paper, we will always use prior variance $(\zeta \delta)^{-1} \vI$, where $\delta > 0$ is the usual weight-decay parameter.

Variational learning for large deep nets faces several challenges, the most prominent being the choice of the posterior distribution family $\mathcal{Q}$. If $\mathcal{Q}$ is chosen as a family of Gaussian distributions with diagonal covariance, then every parameter in the neural network has a mean and a variance. Naively optimizing both the mean and the variance, as is often done with Adam or SGD in the variational Bayesian literature, doubles the number of optimization variables and \revision{is also difficult to implement~\citep{KhNi18}}. 

\subsection{Diagonal Variational Methods and IVON}
It is possible to formulate algorithms for the variational objective in \Cref{eq:varlearn} which take a similar form to the Adam optimizer and do not introduce additional variables, as shown by \citet{KhNi18,ZhSu18} and later scaled to larger models \revision{by \citet{OsSw19,shen2024variational}}. These methods exploit a connection between \revision{VL with} Gaussian posteriors and second-order optimization. The main idea is that the update for the diagonal posterior \revision{precision} takes a similar form to the exponential moving average of the \revision{scale vector that adapts the learning rate} in Adam. Thus, only small modifications to Adam, such as adding the posterior noise to the weights before computing the gradient, are required to target the variational objective. For example, the Variational Online Newton (VON) method of \citet{KhNi18} can be implemented with the following update:
\begin{equation}
   \begin{split}
      \widehat \vh &\gets \widehat \vg \circ (\vparam - \vm) / \vsigma^2, \\
      \vh &\gets \beta_2\,\vh + (1-\beta_2)\,\widehat \vh, \\
      \vm &\gets \vm - \alpha\, (\widehat \vg + \delta\,\vm)/(\vh + \delta), 
   \end{split}
  \label{eq:vonstep}
\end{equation}
where $\widehat \vg$ is the stochastic gradient evaluated at a sample from the posterior $\vparam \sim \mathcal{N}(\vm,\,\mathrm{diag}(\vsigma^2))$ with variance $\vsigma^2 = 1 / (\zeta (\vh + \delta))$ is the posterior variance. \smash{$\widehat \vh$} is the per-parameter curvature estimate obtained via Stein's identity (element-wise product denoted by $\circ$), $\vh$ is the running average of the (positive) curvature with decay rate $\beta_2$, $\delta>0$ is a small prior/regularization constant and $\alpha$ is the learning rate. The mean update for $\vm$ uses the curvature as a preconditioner and a weight-decay term $\delta \vm$ due to the prior.

IVON adds an additional quadratic correction term\revision{~\citep{LiSc20}} to the $\vh$-update to keep the posterior variance positive, as well as momentum and clipping. Both VON and IVON however maintain only a single scalar curvature estimate per parameter (no extra auxiliary variables beyond the usual running average), inject Gaussian posterior noise via the sampled $\vparam$, and target the variational objective in \Cref{eq:varlearn}. This makes it a natural baseline for diagonal variational methods, but as all diagonal variational methods, it cannot track inter-parameter correlations in the posterior covariance.

\subsection{Non-Diagonal Variational Methods}

Extending variational learning to Gaussians with non-diagonal covariances in a scalable manner is non-trivial. There has been earlier work on noisy variants of second-order methods with Kronecker-factored curvature approximations (K-FAC), see \citep{ZhSu18,bae2018eigenvalue}. However, K-FAC methods have not yet been shown to work well for modern language model training, and \revision{the same} is to be expected of their noisy counterparts which target the variational objective. SLANG~\citep{mishkin2018slang} maintains a low-rank plus diagonal posterior, but has not been shown to scale to large problems and comes both with implementation and memory overheads. Similarly, the recent VIKING method~\citep{fadel2025viking} requires expensive projection steps to optimize non-diagonal covariances, and has not been shown to scale to transformer architectures. Bayesian online natural gradient (BONG)~\citep{jones2024bayesian} tracks non-diagonal covariances but has only been run on small MNIST-scale models. Variational methods for matrix Gaussian posteriors~\citep{louizos2016structured,LiKh19b,van2024noether} have so far also only been shown to work on smaller problems and usually require significant changes to training pipelines. Overall, \revision{there is so far no scalable and easy-to-use non-diagonal variational method that is implemented as an optimizer and can train modern neural networks}. Our work aims to fill this gap.

\subsection{Kronecker-Factored Preconditioners and SOAP}

Recently, optimizers with Kronecker-factored preconditioners, including Shampoo~\citep{gupta18shampoo} and SOAP~\citep{vyas2024soap}, have gained attention for language model training.
For instance, a distributed implementation of Shampoo \citep{shi2023distributed} won a recent challenge on optimization for deep learning~\citep{kasimbeg2025accelerating}.
Here, we build on these advances.

Modern practical versions of the Shampoo optimizer, here written for a single 2D parameter matrix $\vParam$ inside a neural network, use the following update:
\begin{equation}
   \begin{split}
      \vL &\gets \beta \vL +  (1 - \beta) \vG  \vG^\top, \\ 
      \vR &\gets \beta \vR +  (1 - \beta)  \vG^\top \vG, \\ 
      \vParam &\gets \vParam - \alpha \vL^{-1/p} \vG \vR^{-1/p},
   \end{split}
\end{equation} where $\vG$ denotes the stochastic minibatch gradient. We do not consider momentum and weight decay for simplicity.
While the original Shampoo uses $p=4$, \revision{recent works~\citep{morwani2024new,vyas2024soap,eschenhagen2025purifying}} suggest using a square root, that is, $p=2$.
\citet{vyas2024soap} show a striking similarity of the above update with $p=2$ to running a step of Adafactor or Adam in the eigenspace obtained from the above running average matrices $\vL$ and $\vR$:
\begin{equation}
   \begin{split}
      \vQ_L &\gets \text{Eig}(\vL), \\
      \vQ_R &\gets \text{Eig}(\vR), \\
      \vParam &\gets \vParam - \alpha \vQ_L \vU \vQ_R^\top,
   \end{split}
  \label{eq:soapspace}
\end{equation}
where $\vU$ is Adam's preconditioned update, that is, the gradient momentum divided by the square root of the running average of squared gradients. A crucial point is that the momentum and preconditioner are built from the projected gradients \smash{$\vG^\circ = \vQ_L^\top \vG \vQ_R$}. 
They refer to the above method as ``{S}hampo{O} with {A}dam in the {P}reconditioner's Eigenbasis'' (SOAP), and show promising performance across various language model pretraining tasks. We summarize SOAP for a single layer in \Cref{alg:soap} \revision{(left)}.

In the following, we show that the SOAP optimizer can be modified to minimize \Cref{eq:varlearn} for a variational family of Gaussian posteriors with structured covariance which we call SOAP-Bubbles. The main idea is to replace the Adam update used in SOAP in the eigenspace spanned by $\vQ_L$ and $\vQ_R$ from \Cref{eq:soapspace} in order to target the variational objective. 

\section{SOAP-Bubble: Posteriors within the SOAP Eigenbasis}
To formalize this idea, we consider a variational family defined within a transformed coordinate space. Rather than learning a diagonal Gaussian directly in the original parameter space $\vparam$, as IVON does, we define a diagonal Gaussian over a transformed variable $\vw$:
\begin{equation}
  \mathcal{Q}_{\text{diag}} = \left \{ \gauss(\vw \mid \vm, \text{diag}(\vsigma^2)) \mid \vm \in \mathbb{R}^P, \vsigma \in \mathbb{R}_+^P  \right \},
\end{equation} 
\revision{where} $\vw$ is connected to the original parameter via $\vparam = \vP \vw$ \revision{and} $\vP$ is assumed to be an orthonormal matrix in the following. \revision{The} corresponding family in the original space has non-diagonal covariances:

\begin{equation}
\mathcal{Q}_{\text{bubble}} = \left\{ \, \gauss(\vparam \mid \vP\vm, \vP \text{diag}(\vsigma^2) \vP^\top) \mid \vm \in \mathbb{R}^P, \vsigma \in \mathbb{R}_+^P \, \right\}.
\end{equation}
 Later, we will pick $\vP$ to be a rotation built from a Kronecker product of SOAP's $\vQ_R$ and $\vQ_L$ matrices, and will therefore refer to elements $q \in \mathcal{Q}_{\text{bubble}}$ as a SOAP-Bubble. 

The main insight to allow for an efficient implementation is that the variational objective for the family $\mathcal{Q}_{\text{bubble}}$ with structured covariance can actually be written over the diagonal family with a transformed loss. The following theorem formalizes this fact.
  \begin{restatable}{thm}{reparamthm}
\label{thm:reparam}
  For the above variational families $\mathcal{Q}_{\text{diag}}$ and $\mathcal{Q}_{\text{bubble}}$, the following holds:
\begin{equation}
  \begin{aligned}
  \eqref{eq:varlearn} \,\, &= \min_{q \in \mathcal{Q}_{\textup{bubble}}} ~ \zeta \, \mathbb{E}_{q(\text{$\vparam$})}[\ell(\vparam)] + \myKL{q(\vparam)}{\gauss(\vparam \mid \vm_0, \vSigma_0)} \\
  &= \min_{q \in \mathcal{Q}_{\textup{diag}}} ~ \zeta \, \mathbb{E}_{q(\text{$\vw$})}[\ell(\vP \vw)] + \myKL{q(\vw)}{\gauss(\vw \mid \vP^\top \vm_0, \vP^\top \vSigma_0 \vP)} \label{eq:wobj}.
  \end{aligned}
\end{equation}
\end{restatable}
The result follows directly from the invariance of the KL under invertible transformations, which is the equality case in the data-processing inequality~\citep{cover1999elements}. A proof is in~\Cref{app:thm1}.

\subsection{Eigenbasis Variational Online Newton (EVON)}

\begin{figure}[t!]
\begin{minipage}{0.5\linewidth}
\input{algorithms/algorithm_soap}
\end{minipage}
\begin{minipage}{0.5\linewidth}
\input{algorithms/algorithm}
\end{minipage}
\caption{EVON entails only simple modifications to the SOAP optimizer (highlighted in \textcolor{TolMutedRed}{\text{red}}). Given an existing implementation of SOAP, it is easy to implement and scales to large problems as SOAP does. $\odot$ denotes element-wise multiplication of matrices, $/$ elementwise division and $^{\odot 2}$ elementwise squaring. As in \citet{vyas2024soap}, the updates are written here for a single 2D parameter matrix.}
\end{figure}

Using \Cref{thm:reparam}, we can obtain our proposed EVON method shown in \Cref{alg:evon}. The method is derived by applying IVON~\citep{LiSc20,shen2024variational} to the variational objective in $q(\vw)$, where we set $\vP = \vQ_R \otimes \vQ_L$ as the Kronecker product of the matrices estimated in SOAP. In particular, we use \citep[Alg.~1]{shen2024variational} and rewrite the iteration in terms of the posterior in the original space~$\vparam$.  The variational objective in $q(\vw)$ then fits exactly the form used in \citep[Eq.~1]{shen2024variational} under a transformed loss \smash{$\tilde \ell(\vw) = \ell (\vP \vw)$}. A full self-contained derivation is given in \Cref{app:derivation}, and the resulting EVON method is shown in \Cref{alg:evon}.

The modifications made in red to SOAP in~\Cref{alg:evon} enable us to estimate a variational posterior with non-diagonal covariance in the original space $\vparam$, that is, $q(\vparam) = \gauss(\vparam \mid \vmu, \vSigma)$ where 
\begin{equation}
   \vmu = \text{vec}(\vM) 
   \quad \text{ and } \quad
   \vSigma = \Big(\vQ_R \otimes \vQ_L\Big) \, \text{diag}\sqr{\text{vec}(\vV)} \, \rnd{\vQ_R^\top \otimes \vQ_L^\top}.
\label{eq:evonpost}
\end{equation}
EVON minimizes \Cref{eq:varlearn} for this family of posteriors, whereas SOAP just targets the regular loss $\ell(\vparam)$. We will now go over all the modifications that EVON makes to SOAP in \Cref{alg:evon}, one by one. 
\begin{enumerate}[leftmargin=16pt,label=(\roman*),labelsep=0.5em,itemsep=0pt, topsep=0pt]
  \item Unlike SOAP, EVON uses random samples drawn from the variational posterior (see \Cref{eq:evonpost}) in lines 1 and 2 and then computes the gradient at the randomly sampled weights in line 3. This yields an unbiased estimator of the gradient of the expected loss $\mathbb{E}_{q(\text{$\vparam$})}[\ell(\vparam)]$ in the variational objective rather than targeting the regular loss $\ell(\vparam)$ as in SOAP.

\item Line 4 changes the squared-gradient to an estimator of the (expected) Hessian. It is known that using squared-gradients approximation to the Hessian within a variational method leads to biased solutions, even on simple logistic regression problems~\citep[Fig.~2a]{KhNi18}. Due to this change, EVON is able to recover good solutions on simple variational Bayesian problems, though a squared-gradient heuristic may also work well for deep learning as it trades-off variance for bias~\citep{Gr11,OsSw19}. To reduce the variance, we added an optional clipping in line~4, which we found empirically to improve the method when training transformers.

\item Line 6 uses the quadratic correction term proposed by~\citet{LiSc20}, which ensures that the posterior precision remains positive. This is needed, since we departed from the squared-gradient estimator to a Hessian estimator, which may not be positive anymore in nonconvex problems. 

\item The preconditioned gradient in line 7 now includes the prior, written through the weight-decay parametrization $\delta$, and we precondition with the Hessian rather than the square-root.

\item Line 8 clips the update, either element-wise as in IVON or Sophia~\citep{liu2023sophia}, or spectrally~\citep{carlson2015preconditioned,jordan2024muon}. We use spectral clipping, which performed considerably better in our experiments (\Cref{app:spectral_clipping}). 

\end{enumerate}

The remaining lines of the two algorithms are completely identical, allowing to implement EVON as a simple modification of SOAP's existing codebase. 
Since SOAP updates both $\vQ_L$ and $\vQ_R$ occasionally during the iteration, the variational family that is being optimized also changes during the course of training. When updating $\vQ_L$ and $\vQ_R$ in line~11, both for SOAP and EVON the momentum $\bar \vG$ is projected into the new space via \smash{$\bar \vG \gets \vQ_L^\top \vQ_L' \bar \vG' (\vQ_R')^\top \vQ_R$}, where the apostrophe denotes the old $\vQ$ and $\bar \vG$. However, both methods do not update $\vH$ into the new space.

\begin{figure}[t]
  \includegraphics[width=\linewidth]{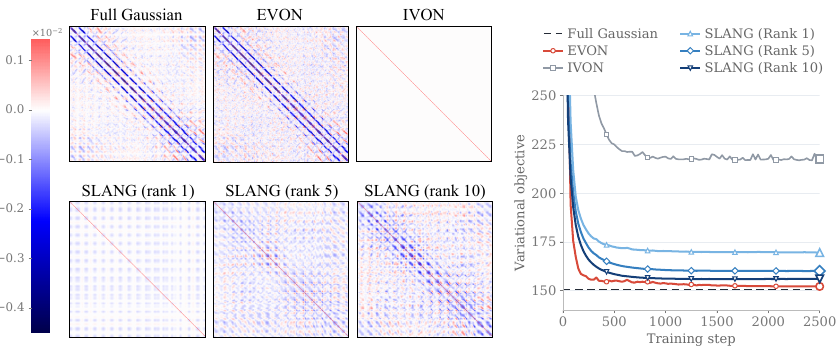}
  \caption{EVON is not only scalable to language models, but also works well for classical variational Bayesian inference. Here, we show that on a binary Bayesian logistic regression, EVON's covariance closely matches the full Gaussian $256\times256$ posterior covariance matrix on the USPS dataset (16x16 grayscale digits, 3vs5). IVON and SLANG~\citep{mishkin2018slang} use diagonal or low-rank plus diagonal approximations which are worse. EVON's variational objective also closely matches the full objective whereas IVON and SLANG need more iterations and converge to a worse value. }
  \label{fig:binary-logreg}
\end{figure}

\subsection{Implementation Details}
\label{sec:impl}
We now discuss several key implementation details which are important to make EVON practical.

\textbf{Distributed Training.} 
 In distributed data-parallel training, using independent noise across GPUs is beneficial as it reduces the variance of these stochastic estimators~\citep{kingma2015variational,OsSw19}.
 EVON requires a careful implementation beyond a naive distributed data-parallel reduction of gradients to ensure the noisy gradients are aggregated consistently for the Hessian estimator.

\textbf{Approximate Eigenbasis.} Both SOAP and EVON use the QR decomposition as an efficient approximation to the eigendecomposition in line 11, as explained in \citet{vyas2024soap}. We use {\tt float32} precision for the (approximate) eigenbasis computation, but the rest can be kept in {\tt bfloat16}.

\textbf{Adaptive Hessian Clipping.} Optionally, we clip the Hessian estimator in line 4 adaptively relative to the current moving average of the Hessian. Specifically, if $\vF$ is the raw Hessian estimator in line 4 at a given step, we apply element-wise clipping to obtain $\text{clip}(\vF, -\gamma(\vH + \epsilon), \gamma(\vH + \epsilon))$, where $\vH$ is the current Hessian estimate, $\gamma$ is a clip ratio (defaulting to $10$), and $\epsilon$ is a small constant.

\textbf{Spectral Clipping.} Spectral clipping sets all singular values of the update in line 8 to roughly one, using a quintic Newton--Schulz iteration as in Muon~\citep{jordan2024muon} instead of an SVD (\Cref{app:spectral_clipping}). 

\begin{figure}[t!]
  \includegraphics[width=\linewidth]{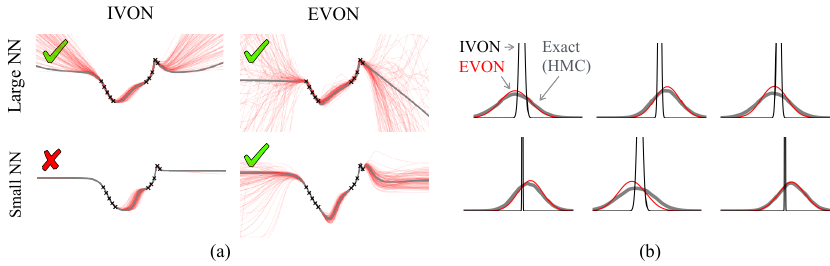}
  \caption{\textbf{(a)} We compare IVON and EVON's posterior samples on a small and a large MLP. While both methods work for large MLPs, for small networks, IVON's uncertainty estimates outside of the data are collapsed. In contrast, EVON gives more diverse posterior samples in regions where the data does not fully specify the solution, showing the strength of non-diagonal posterior covariances. \textbf{(b)}~We~show reference posterior marginal distributions (gray) computed with HMC~\citep{hoffman2014no,carpenter2017stan} on multiclass logistic regression on the UCI Iris dataset against IVON's mean-field approximation (black) and EVON's posterior (red). IVON's mean-field approximation underestimates the variance, while EVON better captures the HMC marginals. } 
  \label{fig:mlp}
\end{figure}

\subsection{Theoretical Guarantees}
Originally, we conceived EVON as a scalable way to obtain non-diagonal covariances for large-scale variational Bayesian training of transformer-based language models. In experiments on simple linear models we found that the method is surprisingly accurate and often finds solutions close to the exact posterior among all multivariate Gaussian posteriors. Motivated by these practical findings, shown in \Cref{fig:binary-logreg}, we provide here a theoretical explanation under which conditions exact solutions are possible. 

The following theorem and corollary characterize sufficient conditions under which this is the case. 

\begin{restatable}{thm}{exactrecovery}
\label{thm:exactrecovery}
    Assume a linear model $f$, loss function $c$, an isotropic prior and a dataset \smash{$\mathcal{D}=\{(\vx_i, y_i)\}_{i=1,\ldots,n}$}. If the matrices $\mathbf{L}$ and $\{\vx_i \vx_i^\top\}_{i=1,\ldots,n}$ and the matrices $\mathbf{R}$ and $\{\mathbb{E}_{\vparam\sim q}[\nabla^2_{f} c(y_i,f(\vParam,\vx_i))]\}_{i=1,\ldots,n}$ are simultaneously diagonalizable, then the optimal posterior $q$ among all multivariate Gaussians is a SOAP-Bubble, that is, in EVON's variational family.
    \label{thm:theory}
~\end{restatable}
\begin{restatable}{corr}{exactrecoverycorr}
\label{corr:exactrecoverycorr}
    Assume a linear model, and an isotropic prior. Assume datapoints/features $\vx_i$ are pairwise orthogonal or parallel, and fix the optimizer batch size to 1. For binary logistic regression and for general linear regression, the exact full Gaussian posterior among all multivariate Gaussians is a SOAP-Bubble, that is, in EVON's variational family.
\end{restatable}

\section{Experiments}
\label{sec:exp}

\subsection{Illustrative Examples}
\label{sec:exp-toy}
We illustrate the intuition behind EVON on a simple logistic regression problem with two-dimensional features in \Cref{fig:fig1}. IVON uses a diagonal posterior covariance, visualized as an ellipse in the figure. This forces the ellipse to be axis aligned which causes the posterior mean to move far away from the mode. EVON's rotated space aligns well with the shape of the posterior, allowing for a rotated ellipse that captures more of the posterior mass. 

In \Cref{fig:mlp} (left), we run IVON and EVON on one large and one small fully-connected neural network on a simple regression task. Due to the limited capacity of the small neural network, the uncertainty estimates from IVON's diagonal posterior lack diversity. EVON's non-diagonal posterior finds highly diverse solutions in regions where there is little data which highlights the benefits of non-diagonal posterior covariances for neural networks. 
Details for both of these experiments are found in~\Cref{app:toy_hyper}.

\subsection{Bayesian Logistic Regression}
\label{sec:exp-logreg}
To study the expressiveness of EVON's variational family\revision{,} we consider a binary Bayesian logistic regression problem on the USPS digit classification task next. The experiment follows the setup of~\citet[Fig.~1]{mishkin2018slang}. We compare EVON to an exact full Gaussian posterior and show the results in \Cref{fig:binary-logreg}. We find that both the covariance structure and variational bound attained match those of the full Gaussian posterior. These results provide empirical evidence of accurate Gaussian approximation beyond our sufficient theoretical conditions in \Cref{thm:theory}. A diagonal posterior used in IVON leads to a worse variational bound and slower convergence. SLANG~\citep{mishkin2018slang} can get closer to the full Gaussian posterior when the rank is sufficiently high but has so far not been shown to scale to larger problems. Rank $10$ SLANG still performs worse than EVON. 

We also consider a multiclass Bayesian logistic regression problem on the UCI Iris dataset on the right hand side of \Cref{fig:mlp}. There, we compute reference posterior marginals using sampling~\citep{hoffman2014no}. EVON's Gaussian marginals closely match these reference marginals, whereas IVON's marginals are highly concentrated. 
We report details for these experiments in~\Cref{app:logreg_hyper}.

\begin{figure}[t!]
  \centering
  \includegraphics[width=\linewidth]{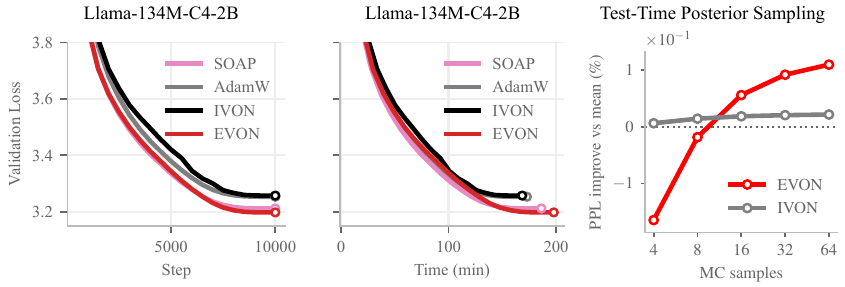}
  \caption{EVON also works well on Llama models, improving significantly over IVON training, similar to how SOAP improves over AdamW. EVON's final validation loss is lower than SOAP's and training with EVON provides a variational posterior on top for free. Averaging next-token probabilities of random models drawn from EVON's SOAP-Bubble before predicting the next token decreases the test-loss more than using IVON's posterior.}
  \label{fig:llama}
\end{figure}

\subsection{Language Model Pre-Training}
\label{sec:exp-lm}
Our goal is to show that EVON can provide good performance on larger scales, for example, for language model pretraining.
To show this, we pretrain language models using NanoGPT~\citep{modded_nanogpt_2024} (123M parameters) in \Cref{fig:fig1} (middle) and the LLaMA-based architecture from~\citet{glentis2025minimal} (134M parameters) with EVON, following the same setup as \citet{lin2026understanding} in \Cref{fig:llama}, and compare it to AdamW, IVON, and SOAP. 
We fix training hyperparameters like the token budget, batch size, learning-rate schedule, and data pipeline across all methods and only swap the optimizer.
We report details for these experiments in~\Cref{app:pretrain_hyper}.

In \Cref{fig:fig1}, we compare IVON and EVON on NanoGPT pre-training on a 1B subset of FineWeb~\citep{fineweb}. EVON demonstrates superior performance by quickly reaching a lower loss in the initial phase of the optimization, and maintains a large gap up until the end of training. We additionally report the perplexity (PPL) improvement relative to each method's posterior mean as a function of the number of Monte Carlo samples used for Bayesian model-averaging on a held-out validation set. For very small sample counts, averaging is not yet beneficial, but with more samples both methods improve over their posterior means. EVON delivers larger gains and scales better with the number of samples, whereas IVON's improvements quickly saturate. This indicates that EVON's non-diagonal posterior provides more useful uncertainty estimates that could be employed further for test-time scaling, where IVON's posterior has already provided improvements~\citep{cong2025improving}.

In \Cref{fig:llama} (left, middle), we compare IVON, SOAP, AdamW and EVON on LLaMA training. We train all models on the same 2B token subset of C4~\citep{2020t5}.
Just as SOAP improves upon AdamW, EVON improves over IVON and reaches an overall lower loss than SOAP, while also providing a variational posterior. The variational posterior is immediately useful to get better results with more inference compute, and \Cref{fig:llama} shows that averaging softmax predictions of posterior ensembles for predicting the next token reduces the overall validation loss. We observe a bigger reduction in test loss for EVON than for IVON, again suggesting that the non-diagonal posterior contains better uncertainty estimates.

To verify that these findings hold at larger scale, we additionally pretrain a 350M LLaMA model under the same protocol (\Cref{fig:llama-350m} in the appendix). EVON again achieves the lowest final validation loss, confirming that the advantage over SOAP and IVON is preserved at increased model scale.

\begin{table}[t!]
\centering
\resizebox{\textwidth}{!}{
\begin{tabular}{llccccccccc}
\toprule
Metric & Optimizer & CARS & DTD & EuroSAT & GTSRB & MNIST & RESISC45 & SUN397 & SVHN & Avg \\
\midrule
\multirow{7}{*}{Acc. $\uparrow$}
 & AdamW & $84.03_{\scriptscriptstyle 2.75}$ & $81.20_{\scriptscriptstyle 1.29}$ & $98.32_{\scriptscriptstyle 0.44}$ & $98.97_{\scriptscriptstyle 0.11}$ & $99.51_{\scriptscriptstyle 0.13}$ & $95.61_{\scriptscriptstyle 0.25}$ & $77.85_{\scriptscriptstyle 0.22}$ & $96.76_{\scriptscriptstyle 0.37}$ & $91.53$ \\
 & IVON & $87.11_{\scriptscriptstyle 0.39}$ & $80.48_{\scriptscriptstyle 0.43}$ & $98.71_{\scriptscriptstyle 0.12}$ & $98.71_{\scriptscriptstyle 0.11}$ & $99.52_{\scriptscriptstyle 0.12}$ & $95.17_{\scriptscriptstyle 0.32}$ & $77.21_{\scriptscriptstyle 0.11}$ & $96.98_{\scriptscriptstyle 0.34}$ & $91.74$ \\
 & ---''--- @32 & $87.32_{\scriptscriptstyle 0.53}$ & $80.60_{\scriptscriptstyle 0.39}$ & $98.78_{\scriptscriptstyle 0.07}$ & $98.74_{\scriptscriptstyle 0.08}$ & $99.53_{\scriptscriptstyle 0.12}$ & $95.23_{\scriptscriptstyle 0.25}$ & $77.27_{\scriptscriptstyle 0.11}$ & $97.01_{\scriptscriptstyle 0.32}$ & $91.81$ \\
 & SOAP & $81.39_{\scriptscriptstyle 4.05}$ & $\mathbf{82.98}_{\scriptscriptstyle 0.89}$ & $\mathbf{98.98}_{\scriptscriptstyle 0.08}$ & $98.40_{\scriptscriptstyle 0.14}$ & $99.55_{\scriptscriptstyle 0.09}$ & $95.74_{\scriptscriptstyle 0.03}$ & $77.51_{\scriptscriptstyle 0.09}$ & $96.92_{\scriptscriptstyle 0.15}$ & $91.43$ \\
 & +BayesVLM & $85.82_{\scriptscriptstyle 5.19}$ & \textcolor{TolMutedRed}{$ 82.70_{\scriptscriptstyle 0.61} $} & \textcolor{TolMutedRed}{$ 98.96_{\scriptscriptstyle 0.05} $} & $98.36_{\scriptscriptstyle 0.31}$ & $99.55_{\scriptscriptstyle 0.03}$ & $\mathbf{95.96}_{\scriptscriptstyle 0.18}$ & $77.61_{\scriptscriptstyle 0.13}$ & $96.94_{\scriptscriptstyle 0.27}$ & $91.99$ \\
\rowcolor{gray!20}\cellcolor{white} & EVON & \textcolor{TolMutedRed}{$ 88.11_{\scriptscriptstyle 0.45} $} & $82.15_{\scriptscriptstyle 0.36}$ & $98.72_{\scriptscriptstyle 0.13}$ & \textcolor{TolMutedRed}{$ 99.02_{\scriptscriptstyle 0.13} $} & \textcolor{TolMutedRed}{$ 99.60_{\scriptscriptstyle 0.06} $} & $95.31_{\scriptscriptstyle 0.82}$ & \textcolor{TolMutedRed}{$ 78.84_{\scriptscriptstyle 0.24} $} & \textcolor{TolMutedRed}{$ 97.35_{\scriptscriptstyle 0.19} $} & \textcolor{TolMutedRed}{$ 92.39 $} \\
\rowcolor{gray!20}\cellcolor{white} & ---''--- @32 & $\mathbf{88.51}_{\scriptscriptstyle 0.47}$ & $81.99_{\scriptscriptstyle 0.26}$ & $98.87_{\scriptscriptstyle 0.11}$ & $\mathbf{99.11}_{\scriptscriptstyle 0.15}$ & $\mathbf{99.63}_{\scriptscriptstyle 0.05}$ & \textcolor{TolMutedRed}{$ 95.79_{\scriptscriptstyle 0.22} $} & $\mathbf{79.28}_{\scriptscriptstyle 0.16}$ & $\mathbf{97.41}_{\scriptscriptstyle 0.16}$ & $\mathbf{92.57}$ \\
\midrule
\multirow{7}{*}{NLL $\downarrow$}
 & AdamW & $0.564_{\scriptscriptstyle 0.086}$ & $0.813_{\scriptscriptstyle 0.043}$ & $0.055_{\scriptscriptstyle 0.012}$ & \textcolor{TolMutedRed}{$ 0.046_{\scriptscriptstyle 0.010} $} & $0.031_{\scriptscriptstyle 0.005}$ & $0.156_{\scriptscriptstyle 0.007}$ & $0.854_{\scriptscriptstyle 0.018}$ & $0.133_{\scriptscriptstyle 0.011}$ & $0.33$ \\
 & IVON & $\mathbf{0.409}_{\scriptscriptstyle 0.010}$ & $0.731_{\scriptscriptstyle 0.006}$ & $0.042_{\scriptscriptstyle 0.006}$ & $0.053_{\scriptscriptstyle 0.007}$ & $0.033_{\scriptscriptstyle 0.007}$ & $0.183_{\scriptscriptstyle 0.008}$ & $0.791_{\scriptscriptstyle 0.006}$ & $0.132_{\scriptscriptstyle 0.015}$ & $0.30$ \\
 & ---''--- @32 & $0.411_{\scriptscriptstyle 0.013}$ & \textcolor{TolMutedRed}{$ 0.670_{\scriptscriptstyle 0.013} $} & $\mathbf{0.034}_{\scriptscriptstyle 0.003}$ & $0.051_{\scriptscriptstyle 0.007}$ & $0.032_{\scriptscriptstyle 0.007}$ & $0.164_{\scriptscriptstyle 0.009}$ & $0.775_{\scriptscriptstyle 0.003}$ & \textcolor{TolMutedRed}{$ 0.129_{\scriptscriptstyle 0.014} $} & \textcolor{TolMutedRed}{$ 0.28 $} \\
 & SOAP & $0.675_{\scriptscriptstyle 0.160}$ & $0.696_{\scriptscriptstyle 0.014}$ & $0.042_{\scriptscriptstyle 0.005}$ & $0.070_{\scriptscriptstyle 0.010}$ & $0.035_{\scriptscriptstyle 0.004}$ & $0.161_{\scriptscriptstyle 0.009}$ & $0.832_{\scriptscriptstyle 0.013}$ & $0.141_{\scriptscriptstyle 0.005}$ & $0.33$ \\
 & +BayesVLM & $0.492_{\scriptscriptstyle 0.176}$ & $0.672_{\scriptscriptstyle 0.010}$ & $0.041_{\scriptscriptstyle 0.003}$ & $0.068_{\scriptscriptstyle 0.018}$ & $0.031_{\scriptscriptstyle 0.003}$ & \textcolor{TolMutedRed}{$ 0.145_{\scriptscriptstyle 0.007} $} & $0.841_{\scriptscriptstyle 0.001}$ & $0.135_{\scriptscriptstyle 0.012}$ & $0.30$ \\
\rowcolor{gray!20}\cellcolor{white} & EVON & $0.415_{\scriptscriptstyle 0.010}$ & $0.787_{\scriptscriptstyle 0.013}$ & $0.044_{\scriptscriptstyle 0.005}$ & $0.059_{\scriptscriptstyle 0.010}$ & \textcolor{TolMutedRed}{$ 0.027_{\scriptscriptstyle 0.002} $} & $0.180_{\scriptscriptstyle 0.010}$ & \textcolor{TolMutedRed}{$ 0.758_{\scriptscriptstyle 0.004} $} & \textcolor{TolMutedRed}{$ 0.129_{\scriptscriptstyle 0.008} $} & $0.30$ \\
\rowcolor{gray!20}\cellcolor{white} & ---''--- @32 & $\mathbf{0.409}_{\scriptscriptstyle 0.008}$ & $\mathbf{0.633}_{\scriptscriptstyle 0.005}$ & \textcolor{TolMutedRed}{$ 0.035_{\scriptscriptstyle 0.002} $} & $\mathbf{0.035}_{\scriptscriptstyle 0.005}$ & $\mathbf{0.026}_{\scriptscriptstyle 0.001}$ & $\mathbf{0.123}_{\scriptscriptstyle 0.007}$ & $\mathbf{0.695}_{\scriptscriptstyle 0.003}$ & $\mathbf{0.117}_{\scriptscriptstyle 0.005}$ & $\mathbf{0.26}$ \\
\midrule
\multirow{7}{*}{\shortstack{ECE $\downarrow$ \\ ($\times 100$)}}
 & AdamW & $4.92_{\scriptscriptstyle 0.60}$ & $9.21_{\scriptscriptstyle 0.69}$ & $0.89_{\scriptscriptstyle 0.12}$ & $\mathbf{0.29}_{\scriptscriptstyle 0.06}$ & $1.01_{\scriptscriptstyle 0.02}$ & $2.21_{\scriptscriptstyle 0.09}$ & $9.03_{\scriptscriptstyle 0.57}$ & $0.78_{\scriptscriptstyle 0.20}$ & $3.54$ \\
 & IVON & $\mathbf{1.67}_{\scriptscriptstyle 0.41}$ & $4.93_{\scriptscriptstyle 0.40}$ & $0.61_{\scriptscriptstyle 0.15}$ & $0.59_{\scriptscriptstyle 0.09}$ & $1.07_{\scriptscriptstyle 0.18}$ & $2.50_{\scriptscriptstyle 0.20}$ & $4.84_{\scriptscriptstyle 0.24}$ & $\mathbf{0.69}_{\scriptscriptstyle 0.15}$ & \textcolor{TolMutedRed}{$ 2.11 $} \\
 & ---''--- @32 & $5.76_{\scriptscriptstyle 0.54}$ & $\mathbf{2.47}_{\scriptscriptstyle 0.52}$ & $\mathbf{0.43}_{\scriptscriptstyle 0.08}$ & $0.60_{\scriptscriptstyle 0.11}$ & $1.09_{\scriptscriptstyle 0.18}$ & \textcolor{TolMutedRed}{$ 1.64_{\scriptscriptstyle 0.22} $} & \textcolor{TolMutedRed}{$ 3.75_{\scriptscriptstyle 0.18} $} & $0.79_{\scriptscriptstyle 0.14}$ & $\mathbf{2.07}$ \\
 & SOAP & $5.46_{\scriptscriptstyle 1.10}$ & $7.04_{\scriptscriptstyle 0.78}$ & $0.71_{\scriptscriptstyle 0.08}$ & $0.75_{\scriptscriptstyle 0.14}$ & \textcolor{TolMutedRed}{$ 0.85_{\scriptscriptstyle 0.04} $} & $2.48_{\scriptscriptstyle 0.15}$ & $7.57_{\scriptscriptstyle 0.68}$ & $\mathbf{0.69}_{\scriptscriptstyle 0.08}$ & $3.19$ \\
 & +BayesVLM & \textcolor{TolMutedRed}{$ 2.73_{\scriptscriptstyle 1.06} $} & $4.94_{\scriptscriptstyle 0.55}$ & $0.73_{\scriptscriptstyle 0.03}$ & \textcolor{TolMutedRed}{$ 0.47_{\scriptscriptstyle 0.05} $} & $\mathbf{0.40}_{\scriptscriptstyle 0.05}$ & $2.11_{\scriptscriptstyle 0.20}$ & $7.54_{\scriptscriptstyle 0.12}$ & $0.85_{\scriptscriptstyle 0.39}$ & $2.47$ \\
\rowcolor{gray!20}\cellcolor{white} & EVON & $3.89_{\scriptscriptstyle 0.22}$ & $8.86_{\scriptscriptstyle 0.26}$ & \textcolor{TolMutedRed}{$ 0.57_{\scriptscriptstyle 0.10} $} & $0.58_{\scriptscriptstyle 0.13}$ & $0.95_{\scriptscriptstyle 0.03}$ & $2.70_{\scriptscriptstyle 0.18}$ & $6.54_{\scriptscriptstyle 0.20}$ & $0.98_{\scriptscriptstyle 0.12}$ & $3.13$ \\
\rowcolor{gray!20}\cellcolor{white} & ---''--- @32 & $6.39_{\scriptscriptstyle 0.61}$ & \textcolor{TolMutedRed}{$ 3.53_{\scriptscriptstyle 0.32} $} & $0.67_{\scriptscriptstyle 0.05}$ & $0.86_{\scriptscriptstyle 0.15}$ & $1.49_{\scriptscriptstyle 0.04}$ & $\mathbf{0.51}_{\scriptscriptstyle 0.17}$ & $\mathbf{1.30}_{\scriptscriptstyle 0.18}$ & $2.73_{\scriptscriptstyle 0.30}$ & $2.19$ \\
\midrule
\multirow{7}{*}{\shortstack{Brier $\downarrow$ \\ ($\times 10$)}}
 & AdamW & $2.35_{\scriptscriptstyle 0.38}$ & $2.92_{\scriptscriptstyle 0.16}$ & $0.26_{\scriptscriptstyle 0.06}$ & \textcolor{TolMutedRed}{$ 0.15_{\scriptscriptstyle 0.02} $} & $0.09_{\scriptscriptstyle 0.02}$ & $0.69_{\scriptscriptstyle 0.03}$ & $3.30_{\scriptscriptstyle 0.05}$ & $0.52_{\scriptscriptstyle 0.06}$ & $1.3$ \\
 & IVON & $1.88_{\scriptscriptstyle 0.07}$ & $2.91_{\scriptscriptstyle 0.04}$ & $0.20_{\scriptscriptstyle 0.02}$ & $0.18_{\scriptscriptstyle 0.02}$ & $0.08_{\scriptscriptstyle 0.02}$ & $0.77_{\scriptscriptstyle 0.04}$ & $3.27_{\scriptscriptstyle 0.01}$ & $0.49_{\scriptscriptstyle 0.06}$ & $1.2$ \\
 & ---''--- @32 & $1.91_{\scriptscriptstyle 0.08}$ & $2.82_{\scriptscriptstyle 0.06}$ & $0.18_{\scriptscriptstyle 0.01}$ & $0.18_{\scriptscriptstyle 0.02}$ & $0.08_{\scriptscriptstyle 0.02}$ & $0.73_{\scriptscriptstyle 0.04}$ & $3.24_{\scriptscriptstyle 0.01}$ & $0.48_{\scriptscriptstyle 0.06}$ & $1.2$ \\
 & SOAP & $2.71_{\scriptscriptstyle 0.56}$ & $2.63_{\scriptscriptstyle 0.08}$ & $\mathbf{0.17}_{\scriptscriptstyle 0.01}$ & $0.24_{\scriptscriptstyle 0.03}$ & $0.08_{\scriptscriptstyle 0.01}$ & $0.67_{\scriptscriptstyle 0.02}$ & $3.29_{\scriptscriptstyle 0.02}$ & $0.51_{\scriptscriptstyle 0.02}$ & $1.3$ \\
 & +BayesVLM & $2.07_{\scriptscriptstyle 0.71}$ & $\mathbf{2.58}_{\scriptscriptstyle 0.05}$ & $\mathbf{0.17}_{\scriptscriptstyle 0.01}$ & $0.24_{\scriptscriptstyle 0.05}$ & $0.08_{\scriptscriptstyle 0.00}$ & \textcolor{TolMutedRed}{$ 0.64_{\scriptscriptstyle 0.02} $} & $3.30_{\scriptscriptstyle 0.00}$ & $0.50_{\scriptscriptstyle 0.04}$ & $1.2$ \\
\rowcolor{gray!20}\cellcolor{white} & EVON & $\mathbf{1.77}_{\scriptscriptstyle 0.05}$ & $2.81_{\scriptscriptstyle 0.02}$ & $0.20_{\scriptscriptstyle 0.01}$ & $0.17_{\scriptscriptstyle 0.02}$ & $\mathbf{0.07}_{\scriptscriptstyle 0.01}$ & $0.71_{\scriptscriptstyle 0.03}$ & \textcolor{TolMutedRed}{$ 3.08_{\scriptscriptstyle 0.01} $} & \textcolor{TolMutedRed}{$ 0.45_{\scriptscriptstyle 0.03} $} & \textcolor{TolMutedRed}{$ 1.2 $} \\
\rowcolor{gray!20}\cellcolor{white} & ---''--- @32 & \textcolor{TolMutedRed}{$ 1.82_{\scriptscriptstyle 0.05} $} & \textcolor{TolMutedRed}{$ 2.60_{\scriptscriptstyle 0.02} $} & $0.18_{\scriptscriptstyle 0.00}$ & $\mathbf{0.13}_{\scriptscriptstyle 0.02}$ & $\mathbf{0.07}_{\scriptscriptstyle 0.00}$ & $\mathbf{0.62}_{\scriptscriptstyle 0.03}$ & $\mathbf{2.96}_{\scriptscriptstyle 0.02}$ & $\mathbf{0.43}_{\scriptscriptstyle 0.02}$ & $\mathbf{1.1}$ \\
\bottomrule
\end{tabular}

}
\vspace{1ex}
\caption{We finetune CLIP ViT-B/16 on 8 classification datasets. EVON is highlighted in \smash{\colorbox{gray!20}{gray}}, the best scores are \textbf{bolded}, and the second best scores are in \textcolor{TolMutedRed}{red}.
When compared to IVON, EVON often provides improvements in terms of Accuracy, Negative Log-Likelihood (NLL), Expected Calibration Error (ECE) and Brier scores. For variational optimizers, we can improve scores by averaging the predictions of 32 model samples (BMA). EVON's gains are often bigger compared to IVON's, suggesting that the estimated posterior is more informative. We also compare to BayesVLM~\citep{baumann2026bayesvlm}, applied post-hoc on SOAP checkpoints (denoted by ``+'').
\label{tab:clip_finetuning_results}}
\end{table}
\subsection{Fine-Tuning CLIP Models}
\label{sec:clip_finetuning}
Finally, we use EVON to finetune CLIP models~\citep{radford2021learning} on several downstream tasks in the visual domain. We initialize the model with the OpenAI ViT-B/16 backbone pre-trained using the AdamW optimizer. We then train the linear layers of the ViT model (99\% of parameters) on 8 visual classification benchmarks.
We list these benchmarks in~\Cref{app:clip_hyper} along with further experimental details. We evaluate all methods in terms of downstream accuracy, negative log-likelihood (NLL), expected calibration error (ECE, \citet{ece}), and Brier score \citep{brier} and compare EVON to AdamW, IVON, SOAP, and BayesVLM~\citep{baumann2026bayesvlm}.

Results are summarized in \Cref{tab:clip_finetuning_results}. EVON performs on par with the best performing optimizers and sometimes outperforms them, especially in terms of NLL, ECE and Brier score. To test the effectiveness of the estimated posteriors, we perform Bayesian Model Averaging, which we report as EVON@32 in the table (indicating that we use 32 Monte Carlo samples). Same applies to IVON@32. We also compare to the post-hoc probabilistic method BayesVLM~\citep{baumann2026bayesvlm}, applied on top of SOAP checkpoints.

As shown, variational optimizers provide comparable or better performance, even with pre-trained checkpoints obtained from a non-variational optimizer (AdamW). EVON is especially effective at improving over the best accuracy of the non-variational counterparts. At the same time, when using Bayesian Model Averaging, calibration and uncertainty quantification performance remains strong. For IVON, averaging predictions only brings marginal improvements compared to the mean, perhaps due to lack of expressiveness of the diagonal covariances. EVON@32 achieves the highest average accuracy and the lowest average NLL and Brier score among the evaluated methods. Calibration results are mixed: IVON@32 has the lowest average ECE, and Bayesian model averaging does not improve ECE on every dataset. Following \cite{baan2022stop}, these deviations may stem from the reliance on hard labels or finite-sample bias. We find that ECE alone can be misleading when assessing calibration: posterior averaging reduces maximum-class confidence, which can worsen standard ECE even when proper scoring rules such as NLL and Brier score improve~\citep{guo2017calibration}. Implementation details can be found in~\Cref{app:clip_hyper}.
\section{Discussion and Limitations}
Scaling variational learning with structured weight uncertainty to large neural networks has traditionally been a hard problem. Our proposed EVON method offers a simple solution, requiring only a few modifications to the contemporary second-order deep learning optimizer SOAP. Our results further confirm a fundamental connection between second-order optimization and Gaussian variational Bayesian methods~\revision{\citep{OpAr09,KhNi18,wuThesis}}. Thus, further advances in second-order optimization are expected to lead to better variational learning methods, and vice versa, variational Bayesian principles can \revision{help design} better optimizers.

Despite being scalable to large problems, EVON also works surprisingly well on classical Bayesian inference problems, and we have provided some theoretical justifications for that in \Cref{thm:exactrecovery} and \Cref{corr:exactrecoverycorr}. Still, these rely on simultaneous-diagonalization and orthogonality assumptions that are unlikely to hold exactly in practical deep-learning settings. We view these results primarily as intuition: they show that EVON's structured covariance family is expressive enough to closely match the optimal full-Gaussian variational posterior in controlled settings, and they help explain the strong empirical performance observed in larger-scale problems where the assumptions no longer apply. A general characterization of the objective under a moving eigenbasis remains an open question.

While our focus is on scaling variational Bayesian training to large deep-learning models, the non-deep-learning experiments we conducted suggest that EVON often approaches a reference full-Gaussian solution. Applying EVON-like methods to classical Bayesian models, such as hierarchical models, Gaussian processes, or state-space models, is a promising direction for future work.

For language model training, EVON can even surpass SOAP's final validation loss while offering additional SOAP-Bubble posteriors for free. These can be used to improve validation loss during test-time \revision{and could also be useful in} downstream applications such as continual learning~\citep{NgLi18}, model understanding~\citep{nickl2023memory}, language generation~\citep{daheim2025uncertainty} or exploration and data collection in open-ended settings~\citep{silver2025welcome}\revision{. These are all promising avenues} for future research. 

In distributed settings, EVON requires a careful implementation to minimize computational overhead compared to Adam. Since EVON is structurally similar to SOAP, this can be addressed by building on top of improvements over SOAP or Shampoo~\citep{shi2023distributed,lin2026understanding}. Like any other variational learning method, EVON also introduces an extra hyperparameter, $\zeta$, from the variational objective that must be tuned carefully. Setting $\zeta$ too large causes the posterior to collapse, while setting it too small leads to unstable training. Developing scaling laws to predict how to optimally choose $\zeta$ for a given model and dataset size remains an important direction for future work.

\clearpage

\section*{Acknowledgements}
We thank Cong Bai for helpful discussions and help with implementing an early version of EVON. 
This work is supported by the Bayes duality project, JST CREST Grant Number JPMJCR2112. Thomas Moellenhoff~acknowledges the support of JSPS KAKENHI Grant Number 26H02541.  This research work has been funded by the German Federal Ministry of Research, Technology and Space and the Hessian Ministry of Higher Education, Research, Science and the Arts within their joint support of the National Research Center for Applied Cybersecurity ATHENE. 
Adrian Robert Minut acknowledges support from Sapienza University of Rome, under Startup Research project TVT (Task Vector Transfer across heterogeneous models).

\section*{Author Contribution}
Authors: Adrian R. Minut (ARM), Nico Carlsen Daheim (ND), Marco Miani (MM), Mohammad Emtiyaz Khan (MEK), Wu Lin (WL), Thomas Moellenhoff (TM).

MEK and TM discussed the possibility of variational training in a subspace. This led TM to develop a first version of the EVON algorithm and run some proof-of-concept experiments. With feedback from WL, ND and TM, ARM refined the method and implemented a scalable variant. ARM conducted all the experiments on Llama, nanoGPT and ViT with help from ND, WL and TM. TM did the small-scale illustrative experiments with suggestions from MEK. MM derived the theoretical results, with feedback from TM. All authors contributed to writing the paper, as well as proof-reading.

\bibliography{references}

@string{aaai = {{AAAI} Conference on Artificial Intelligence (AAAI)}}

@string{aistats = {International Conference on Artificial Intelligence and Statistics (AISTATS)}}

@string{cvpr = {IEEE Conference on Computer Vision and Pattern Recognition (CVPR)}}

@string{iclr = {International Conference on Learning Representations (ICLR)}}

@string{icml = {International Conference on Machine Learning (ICML)}}

@string{jmlr = {J. Mach. Learn. Res. (JMLR)}}

@string{emnlp = {Conference on Empirical Methods in Natural Language Processing (EMNLP)}}

@string{nips = {Advances in Neural Information Processing Systems (NeurIPS)}}

@inproceedings{eschenhagen2025purifying,
title={Purifying {S}hampoo: Investigating {S}hampoo's Heuristics by Decomposing its Preconditioner},
author={Runa Eschenhagen and Aaron Defazio and Tsung-Hsien Lee and Richard E. Turner and Hao-Jun Michael Shi},
booktitle=nips, 
year={2025},
url={https://openreview.net/forum?id=kePsKwxvaV}
}

@inproceedings{morwani2024new,
    title={A New Perspective on {S}hampoo's Preconditioner},
    author={Depen Morwani and Itai Shapira and Nikhil Vyas and Eran Malach and Sham M Kakade and Lucas Janson},
    booktitle=iclr, 
    year={2025},
    url={https://openreview.net/forum?id=c6zI3Cp8c6},
}

@InProceedings{gupta18shampoo,
title={{S}hampoo: Preconditioned Stochastic Tensor Optimization},
  author = 	 {Gupta, Vineet and Koren, Tomer and Singer, Yoram},
  booktitle = 	 icml, 
  year = 	 {2018},
}

@article{shi2023distributed,
  title={A Distributed Data-Parallel {P}y{T}orch Implementation of the Distributed {S}hampoo Optimizer for Training Neural Networks At-Scale},
  author={Shi, Hao-Jun Michael and Lee, Tsung-Hsien and Iwasaki, Shintaro and Gallego-Posada, Jose and Li, Zhijing and Rangadurai, Kaushik and Mudigere, Dheevatsa and Rabbat, Michael},
  journal={arXiv preprint arXiv:2309.06497},
  year={2023},
  doi={10.48550/arxiv.2309.06497},
}

@inproceedings{shen2024variational,
  title={Variational learning is effective for large deep networks},
  author={Shen, Yuesong and Daheim, Nico and Cong, Bai and Nickl, Peter and Marconi, Gian Maria and Bazan, Clement and Yokota, Rio and Gurevych, Iryna and Cremers, Daniel and Khan, Mohammad Emtiyaz and M{\"o}llenhoff, Thomas},
  booktitle=icml, 
  year={2024}
}

@inproceedings{guo2017calibration,
	author = {Guo, Chuan and Pleiss, Geoff and Sun, Yu and Weinberger, Kilian Q},
	booktitle = icml,
	title = {On calibration of modern neural networks},
	year = {2017}}

@inproceedings{baumann2026bayesvlm,
 author = {Baumann, Anton and Li, Rui and Klasson, Marcus and Mentu, Santeri and Karthik, Shyamgopal and Akata, Zeynep and Solin, Arno and Trapp, Martin},
 booktitle = iclr,
 title = {Post-hoc Probabilistic Vision-Language Models},
 year = {2026}
}

@inproceedings{ImBa21,
	author = {Immer, Alexander and Bauer, Matthias and Fortuin, Vincent and R{\"a}tsch, Gunnar and Khan Mohammad, Emtiyaz},
	booktitle = icml,
	title = {Scalable marginal likelihood estimation for model selection in deep learning},
	year = {2021}}

@book{Mu12,
	author = {Murphy, Kevin Patrick},
	publisher = {MIT Press},
	title = {Machine learning: a probabilistic perspective},
	year = {2012}}

@inproceedings{LiKh19b,
	author = {Lin, Wu and Khan, Mohammad Emtiyaz and Schmidt, Mark},
	booktitle = icml,
	title = {Fast and simple natural-gradient variational inference with mixture of exponential-family approximations},
	year = {2019}}

@inproceedings{NgLi18,
	author = {Cuong V. Nguyen and Yingzhen Li and Thang D. Bui and Richard E. Turner},
	booktitle = iclr,
	title = {Variational Continual Learning},
	year = {2018}}

@inproceedings{Gr11,
	author = {Alex Graves},
	booktitle = nips,
	title = {Practical Variational Inference for Neural Networks},
	year = {2011}}

@inproceedings{LiSc20,
	author = {Lin, Wu and Schmidt, Mark and Khan, Mohammad Emtiyaz},
	booktitle = icml,
	title = {Handling the Positive-Definite Constraint in the {B}ayesian Learning Rule},
	year = {2020}}

@phdthesis{wuThesis,
	author = {Lin, Wu},
	school = {{University of British Columbia}},
	title = {Computationally efficient geometric methods for optimization and inference in machine learning},
	year = {2023}}

@misc{LiKh19,
	author = {Wu Lin and Mohammad Emtiyaz Khan and Mark Schmidt},
	title = {Stein's Lemma for the Reparameterization Trick with Exponential Family Mixtures},
	year = {2019}}

@article{OpAr09,
	author = {Manfred Opper and C{\'{e}}dric Archambeau},
	journal = {Neural Computation},
	number = {3},
	pages = {786--792},
	title = {The Variational {G}aussian Approximation Revisited},
	volume = {21},
	year = {2009}}

@article{OsSw19,
	author = {Osawa, Kazuki and Swaroop, Siddharth and Jain, Anirudh and Eschenhagen, Runa and Turner, Richard E and Yokota, Rio and Khan, Mohammad Emtiyaz},
	journal = nips,
	title = {Practical deep learning with {B}ayesian principles},
	year = {2019}}

@inproceedings{KhNi18,
	author = {Khan, Mohammad Emtiyaz and Nielsen, Didrik and Tangkaratt, Voot and Lin, Wu and Gal, Yarin and Srivastava, Akash},
	booktitle = icml,
	title = {Fast and scalable {B}ayesian deep learning by weight-perturbation in {A}dam},
	year = {2018}}

@inproceedings{ZhSu18,
	author = {Zhang, Guodong and Sun, Shengyang and Duvenaud, David and Grosse, Roger},
	booktitle = icml,
	title = {Noisy natural gradient as variational inference},
	year = {2018}}

@article{KhRu23,
	author = {Khan, Mohammad Emtiyaz and Rue, Haavard},
	journal = jmlr,
	number = {281},
	pages = {1--46},
	title = {The {B}ayesian Learning Rule},
	volume = {24},
	year = {2023}}

@inproceedings{daxberger2021laplace,
  title={Laplace Redux--Effortless {B}ayesian Deep Learning},
  author={Erik Daxberger and Agustinus Kristiadi and Alexander Immer
          and Runa Eschenhagen and Matthias Bauer and Philipp Hennig},
  booktitle=nips,
  year={2021}
}

@article{khan2025knowledge,
  title={Knowledge Adaptation as Posterior Correction},
  author={Khan, Mohammad Emtiyaz},
  journal={arXiv:2506.14262},
  year={2025}
}

@inproceedings{daheim2025uncertainty,
  title={Uncertainty-aware decoding with minimum {B}ayes risk},
  author={Daheim, Nico and Meister, Clara and M{\"o}llenhoff, Thomas and Gurevych, Iryna},
  booktitle=iclr, 
  year={2025}
}

@inproceedings{nickl2023memory,
  title={The memory-perturbation equation: Understanding model's sensitivity to data},
  author={Nickl, Peter and Xu, Lu and Tailor, Dharmesh and M{\"o}llenhoff, Thomas and Khan, Mohammad Emtiyaz},
  booktitle=nips,
  year={2023}
}

@article{cong2025improving,
  title={Improving {LoRA} with variational learning},
  author={Cong, Bai and Daheim, Nico and Shen, Yuesong and Yokota, Rio and Khan, Mohammad Emtiyaz and M{\"o}llenhoff, Thomas},
  journal={arXiv:2506.14280},
  year={2025}
}

@inproceedings{mollenhoff2025federated,
  title={Federated {ADMM} from {B}ayesian Duality},
  author={M{\"o}llenhoff, Thomas and Swaroop, Siddharth and Doshi-Velez, Finale and Khan, Mohammad Emtiyaz},
  booktitle=iclr, 
  year={2025}
}

@inproceedings{bae2018eigenvalue,
  title={Eigenvalue corrected noisy natural gradient},
  author={Bae, Juhan and Zhang, Guodong and Grosse, Roger},
  booktitle={Neural Information Processing Systems (Bayesian Deep Learning Workshop)},
  year={2018}
}

@inproceedings{mishkin2018slang,
  title={{SLANG}: Fast structured covariance approximations for {B}ayesian deep learning with natural gradient},
  author={Mishkin, Aaron and Kunstner, Frederik and Nielsen, Didrik and Schmidt, Mark and Khan, Mohammad Emtiyaz},
  booktitle=nips, 
  year={2018}
}

@InProceedings{pmlr-v139-lin21e,
  title = 	 {Tractable structured natural-gradient descent using local parameterizations},
  author =       {Lin, Wu and Nielsen, Frank and Emtiyaz, Khan Mohammad and Schmidt, Mark},
  booktitle = 	 icml, 
  year = 	 {2021},
}

@inproceedings{vyas2024soap,
    title={{SOAP}: Improving and Stabilizing Shampoo using Adam for Language Modeling},
    author={Nikhil Vyas and Depen Morwani and Rosie Zhao and Itai Shapira and David Brandfonbrener and Lucas Janson and Sham M. Kakade},
    booktitle=iclr,
    year={2025}
}

@inproceedings{jones2024bayesian,
  title={Bayesian online natural gradient ({BONG})},
  author={Jones, Matt and Chang, Peter and Murphy, Kevin},
  booktitle=nips, 
  year={2024}
}

@book{cover1999elements,
  title={Elements of information theory},
  author={Cover, Thomas M and Thomas, Joy A},
  year={2006},
  publisher={John Wiley \& Sons}
}

@inproceedings{fadel2025viking,
  title={{VIKING}: Deep variational inference with stochastic projections},
  author={Fadel, Samuel G and Roy, Hrittik and Kr{\"a}mer, Nicholas and Zainchkovskyy, Yevgen and Syrota, Stas and Mahou, Alejandro Valverde and Ek, Carl Henrik and Hauberg, S{\o}ren},
  booktitle=nips,
  year={2025}
}

@inproceedings{
daheim2024model,
title={Model Merging by Uncertainty-Based Gradient Matching},
author={Nico Daheim and Thomas M{\"o}llenhoff and Edoardo Ponti and Iryna Gurevych and Mohammad Emtiyaz Khan},
booktitle=iclr,
year={2024}
}

@inproceedings{wang2024blob,
  title={Blob: Bayesian low-rank adaptation by backpropagation for large language models},
  author={Wang, Yibin and Shi, Haizhou and Han, Ligong and Metaxas, Dimitris and Wang, Hao},
  booktitle=nips, 
  year={2024}
}

@inproceedings{dhahri2024shaving,
  title={Shaving weights with {O}ccam's razor: {B}ayesian sparsification for neural networks using the marginal likelihood},
  author={Dhahri, Rayen and Immer, Alexander and Charpentier, Betrand and G{\"u}nnemann, Stephan and Fortuin, Vincent},
  booktitle=nips, 
  year={2024}
}

@inproceedings{louizos2016structured,
  title={Structured and efficient variational deep learning with matrix {G}aussian posteriors},
  author={Louizos, Christos and Welling, Max},
  booktitle=icml, 
  year={2016},
}

@inproceedings{van2024noether,
  title={Noether's razor: Learning conserved quantities},
  author={van der Ouderaa, Tycho F and van der Wilk, Mark and De Haan, Pim},
  booktitle=nips, 
  year={2024}
}

@inproceedings{eschenhagen2023kronecker,
  title={Kronecker-factored approximate curvature for modern neural network architectures},
  author={Eschenhagen, Runa and Immer, Alexander and Turner, Richard and Schneider, Frank and Hennig, Philipp},
  booktitle=nips, 
  year={2023}
}

@inproceedings{hong2025better,
  title={Better Hessians Matter: Studying the Impact of Curvature Approximations in Influence Functions},
  author={Hong, Dat Minh and Mlodozeniec, Bruno Kacper and Eschenhagen, Runa and Turner, Richard E},
  booktitle={Mechanistic Interpretability Workshop at NeurIPS 2025},
  year={2025}
}

@article{gan2026neural,
  title={Neural Thickets: Diverse Task Experts Are Dense Around Pretrained Weights},
  author={Gan, Yulu and Isola, Phillip},
  journal={arXiv:2603.12228},
  year={2026}
}

@misc{jordan2024muon,
  title={Muon: An optimizer for hidden layers in neural networks},
  author={Jordan, Keller and Jin, Yuchen and Boza, Vlado and You, Jiacheng and Cesista, Franz and Newhouse, Laker and Bernstein, Jeremy},
  url={https://kellerjordan.github.io/posts/muon},
  year={2024}
}

@inproceedings{lin2026understanding,
title={Understanding and improving {S}hampoo and {SOAP} via {K}ullback-{L}eibler Minimization},
author={Wu Lin and Scott C. Lowe and Felix Dangel and Runa Eschenhagen and Zikun Xu and Roger Baker Grosse},
booktitle=iclr,
year={2026},
url={https://openreview.net/forum?id=pQQuC1nIQq}
}

@inproceedings{carlson2015preconditioned,
  title={Preconditioned spectral descent for deep learning},
  author={Carlson, David E and Collins, Edo and Hsieh, Ya-Ping and Carin, Lawrence and Cevher, Volkan},
  booktitle=nips, 
  year={2015}
}

@inproceedings{kasimbeg2025accelerating,
  title={Accelerating neural network training: An analysis of the {A}lgo{P}erf competition},
  author={Kasimbeg, Priya and Schneider, Frank and Eschenhagen, Runa and Bae, Juhan and Sastry, Chandramouli Shama and Saroufim, Mark and Feng, Boyuan and Wright, Less and Yang, Edward Z and Nado, Zachary and others},
  booktitle=iclr, 
  year={2025}
}

@inproceedings{huggins2020validated,
  title={Validated variational inference via practical posterior error bounds},
  author={Huggins, Jonathan and Kasprzak, Mikolaj and Campbell, Trevor and Broderick, Tamara},
  booktitle=aistats,
  year={2020},
}

@inproceedings{liu2023sophia,
  title={Sophia: A scalable stochastic second-order optimizer for language model pre-training},
  author={Liu, Hong and Li, Zhiyuan and Hall, David and Liang, Percy and Ma, Tengyu},
  booktitle=iclr,
  year={2024}
}

@article{silver2025welcome,
  title={Welcome to the era of experience},
  author={Silver, David and Sutton, Richard S},
  journal={Google AI},
  volume={1},
  pages={11},
  year={2025}
}

@inproceedings{kingma2015variational,
  title={Variational dropout and the local reparameterization trick},
  author={Kingma, Diederik P and Salimans, Tim and Welling, Max},
  booktitle=nips, 
  year={2015}
}

@inproceedings{cars,
  author={Krause, Jonathan and Stark, Michael and Deng, Jia and Fei-Fei, Li},
  booktitle={2013 IEEE International Conference on Computer Vision Workshops}, 
  title={3{D} Object Representations for Fine-Grained Categorization}, 
  year={2013},
  month={Dec},
}

@inproceedings{dtd,
  Author    = {M. Cimpoi and S. Maji and I. Kokkinos and S. Mohamed and A. Vedaldi},
  Title     = {Describing Textures in the Wild},
  Booktitle = cvpr,
  Year      = {2014}
}

@article{eurosat,
  title={Euro{SAT}: A novel dataset and deep learning benchmark for land use and land cover classification},
  author={Helber, Patrick and Bischke, Benjamin and Dengel, Andreas and Borth, Damian},
  journal={IEEE Journal of Selected Topics in Applied Earth Observations and Remote Sensing},
  volume={12},
  number={7},
  pages={2217--2226},
  year={2019},
  publisher={IEEE}
}

@article{gtsrb,
title = {Man vs. computer: Benchmarking machine learning algorithms for traffic sign recognition},
journal = {Neural Networks},
volume = {32},
pages = {323-332},
year = {2012},
author = {J. Stallkamp and M. Schlipsing and J. Salmen and C. Igel}
}

@article{resisc45,
  title={Remote Sensing Image Scene Classification: Benchmark and State of the Art},
  volume={105},
  ISSN={1558-2256},
  url={http://dx.doi.org/10.1109/JPROC.2017.2675998},
  DOI={10.1109/jproc.2017.2675998},
  number={10},
  journal={Proceedings of the IEEE},
  publisher={Institute of Electrical and Electronics Engineers (IEEE)},
  author={Cheng, Gong and Han, Junwei and Lu, Xiaoqiang},
  year={2017},
  month={Oct},
  pages={1865-1883}
}

@article{sun397,
  title     = {SUN Database: Exploring a Large Collection of Scene Categories},
  author    = {Jianxiong Xiao and Krista A. Ehinger and James Hays and Antonio Torralba and Aude Oliva},
  year      = 2014,
  journal   = {International Journal of Computer Vision},
  volume    = 119,
  pages     = {3--22},
}

@inproceedings{svhn,
  title={Reading digits in natural images with unsupervised feature learning},
  author={Netzer, Yuval and Wang, Tao and Coates, Adam and Bissacco, Alessandro and Wu, Bo and Ng, Andrew Y},
  booktitle={NIPS workshop on deep learning and unsupervised feature learning},
  volume={2011},
  number={2},
  pages={5},
  year={2011}
}

@ARTICLE{brier,
  author = {{Brier}, Glenn W.},
  title = "{Verification of Forecasts Expressed in Terms of Probability}",
  journal = {Monthly Weather Review},
  year = 1950,
  month = jan,
  volume = {78},
  number = {1},
  pages = {1--3},
  doi = {10.1175/1520-0493(1950)078<0001:VOFEIT>2.0.CO;2},
}

@inproceedings{ece,
  title={Obtaining Well Calibrated Probabilities Using {B}ayesian Binning},
  author={Naeini, Mahdi Pakdaman and Cooper, Gregory and Hauskrecht, Milos},
  booktitle=aaai,
  volume={29},
  number={1},
  year={2015},
}

@inproceedings{fineweb,
  title={The fineweb datasets: Decanting the web for the finest text data at scale},
  author={Penedo, Guilherme and Kydl{\'\i}{\v{c}}ek, Hynek and Lozhkov, Anton and Mitchell, Margaret and Raffel, Colin and Von Werra, Leandro and Wolf, Thomas and others},
  booktitle=nips,
  year={2024}
}

@inproceedings{glentis2025minimal,
    title={A Minimalist Optimizer Design for {LLM} Pretraining},
    author={Athanasios Glentis and Jiaxiang Li and Andi Han and Mingyi Hong},
    booktitle={ES-FoMo III: 3rd Workshop on Efficient Systems for Foundation Models},
    year={2025},
    url={https://openreview.net/forum?id=QBJvaWGj04}
}

@misc{modded_nanogpt_2024,
  author       = {Keller Jordan and Jeremy Bernstein and Brendan Rappazzo and
                  @fernbear.bsky.social and Vlado Boza and Jiacheng You and
                  Franz Cesista and Braden Koszarsky and @Grad62304977},
  title        = {modded-nano{GPT}: Speedrunning the {NanoGPT} baseline},
  year         = {2024},
  url          = {https://github.com/KellerJordan/modded-nanogpt}
}

@article{carpenter2017stan,
  title={Stan: A probabilistic programming language},
  author={Carpenter, Bob and Gelman, Andrew and Hoffman, Matthew D and Lee, Daniel and Goodrich, Ben and Betancourt, Michael and Brubaker, Marcus and Guo, Jiqiang and Li, Peter and Riddell, Allen},
  journal={Journal of Statistical Software},
  volume={76},
  pages={1--32},
  year={2017}
}

@article{hoffman2014no,
  title={The {N}o-{U}-{T}urn {S}ampler: Adaptively setting path lengths in {H}amiltonian {M}onte {C}arlo},
  author={Hoffman, Matthew D and Gelman, Andrew},
  journal={J. Mach. Learn. Res.},
  volume={15},
  number={1},
  pages={1593--1623},
  year={2014}
}

@inproceedings{2020t5,
  author  = {Colin Raffel and Noam Shazeer and Adam Roberts and Katherine Lee and Sharan Narang and Michael Matena and Yanqi Zhou and Wei Li and Peter J. Liu},
  title   = {Exploring the Limits of Transfer Learning with a Unified Text-to-Text Transformer},
  booktitle=jmlr,
  year    = {2020},
  volume  = {21},
  number  = {140},
  pages   = {1-67}
}

@inproceedings{radford2021learning,
  title={Learning transferable visual models from natural language supervision},
  author={Radford, Alec and Kim, Jong Wook and Hallacy, Chris and Ramesh, Aditya and Goh, Gabriel and Agarwal, Sandhini and Sastry, Girish and Askell, Amanda and Mishkin, Pamela and Clark, Jack and others},
  booktitle=icml,
  year={2021}
}

@article{sehnke2010parameter,
  title={Parameter-exploring policy gradients},
  author={Sehnke, Frank and Osendorfer, Christian and R{\"u}ckstie{\ss}, Thomas and Graves, Alex and Peters, Jan and Schmidhuber, J{\"u}rgen},
  journal={Neural Networks},
  volume={23},
  number={4},
  pages={551--559},
  year={2010},
}

@inproceedings{baan2022stop,
    title = "Stop Measuring Calibration When Humans Disagree",
    author = "Baan, Joris  and
      Aziz, Wilker  and
      Plank, Barbara  and
      Fernandez, Raquel",
    editor = "Goldberg, Yoav  and
      Kozareva, Zornitsa  and
      Zhang, Yue",
    booktitle = emnlp,
    month = dec,
    year = "2022",
    address = "Abu Dhabi, United Arab Emirates",
    publisher = "Association for Computational Linguistics",
    url = "https://aclanthology.org/2022.emnlp-main.124/",
    doi = "10.18653/v1/2022.emnlp-main.124",
    pages = "1892--1915"
}
\clearpage
\appendix

\section{Proof of \Cref{thm:reparam}}
\label{app:thm1}
\reparamthm*
\begin{proof}
  The main idea is that any $q \in \mathcal{Q}_{\textup{bubble}}$ can be written as a linear transformation (pushforward) $\vP_\sharp q'$ for a distribution $q' \in \mathcal{Q}_{\textup{diag}}$. To see this, recall that a linear transformation of a Gaussian is another Gaussian, that is, \smash{$\vP_\sharp \gauss(\vm', \vSigma') =  \gauss(\vP \vm', \vP \vSigma' \vP^\top)$}. Since any member of $\mathcal{Q}_{\textup{bubble}}$ is obtained that way, we can rewrite the variational objective in~\Cref{eq:varlearn} over the 
  structured family as a minimization over the pushforward of the diagonal family:
\begin{equation}
  \eqref{eq:varlearn} \,\, = \min_{q' \in \mathcal{Q}_{\textup{diag}}} ~ \zeta \, \mathbb{E}_{\text{$\vparam$} \sim \text{\vP}_\sharp q'}[\ell(\vparam)] + \myKL{\vP_\sharp q'}{\gauss(\vm_0, \vSigma_0)}.
\end{equation}
Using the fact that the pushforward is adjoint to the pullback, we can rewrite the expected loss as $\mathbb{E}_{\text{$\vparam$} \sim \text{\vP}_\sharp q'}[\ell(\vparam)] = \mathbb{E}_{\text{\vw} \sim q'}[\ell(\vP \vw)]$, which gives us the first part of the statement. For the KL-term, we use the fact that the KL-divergence is invariant under invertible transformations $T_\sharp$ of both measures, that is: $\myKL{\vP_\sharp q'}{\gauss(\vm_0, \vSigma_0)} = \myKL{\vT_\sharp \vP_\sharp q'}{\vT_\sharp \gauss(\vm_0, \vSigma_0)}$. This is the equality case in the data-processing inequality, when all information is preserved due to invertibility \citep{cover1999elements}. The result also holds for more general divergences, and for self-contained proof, see for example \citet[Lemma~4.1]{huggins2020validated}. Setting $\vT = \vP^{-1} = \vP^\top$ and \revision{using} $\vT_\sharp \vP_\sharp = \text{id}_\sharp$\revision{,} we get, 
\begin{align}
\myKL{\vT_\sharp \vP_\sharp q'}{\vT_\sharp \gauss(\vm_0, \vSigma_0)} &= \myKL{\text{id}_\sharp q'}{\vP^\top_\sharp \gauss(\vm_0, \vSigma_0)} \\
&= \myKL{q'}{\gauss(\vP^\top \vm_0, \vP^\top \vSigma_0 \vP)}, 
\end{align}
which concludes the proof. 
\end{proof}

\section{Derivation of the EVON Algorithm}
\label{app:derivation}
In this appendix, we derive the proposed EVON method and show how it leads to the SOAP-like updates of \Cref{alg:evon}. To derive EVON, we start from the variational online Newton (VON) method in \citet[Eq.~12]{KhRu23} applied to the variational objective in $q(\vw)$, which is \Cref{eq:wobj} in the main text. The VON method directly targets \Cref{eq:wobj}. For a generic Gaussian family $\gauss(\vparam \mid \vm, \text{diag}(\vs)^{-1})$ with diagonal covariance and isotropic Gaussian prior, $\vm_0 = 0$ and $\vSigma_0 = s_0^{-1} \vI$ the method reads\revision{, with $\vone$ the all-ones vector}: 
\begin{align}
    \vs &\gets (1 - \rho) \vs + \rho \left[ \zeta \mathbb{E}_{q(\vparam)}[\text{diag}(\nabla^2 \ell(\vparam))] + \revision{s_0 \vone} \right], \label{eq:von1} \\
    \vm &\gets \vm - \rho \text{diag}(\vs)^{-1} \left(\zeta \mathbb{E}_{q(\vparam)}[\nabla \ell(\vparam)] + s_0 \vm \right), \label{eq:von2}
\end{align}
where $\rho > 0$ is a learning rate.

Now, we apply the above VON updates to the reparametrized problem in $\vw$ from \Cref{thm:reparam}, which leads to the following updates:
\begin{align}
  \vs' &\gets (1 - \rho) \vs' + \rho \left(\zeta \mathbb{E}_{q(\text{$\vw$})}[\text{diag}(\vP^\top \nabla^2 \ell(\vP \vw) \vP)] + \revision{s_0 \vone} \right), \\
  \vm' &\gets \vm' - \rho \, \revision{\text{diag}(\vs')^{-1}} \left(\zeta \vP^\top \mathbb{E}_{q(\text{$\vw$})}[\nabla \ell(\vP \vw)] + s_0 \vm' \right).
\end{align}
Here, we used the chain-rule for gradient and Hessian for the $\ell(\vP \vw)$. We will now use Price's theorem~\citep[Theorem~4]{LiKh19} to write the Hessian in terms of gradients. Also using stochastic gradients, both for minibatches and the expectations, and writing $\vh = (\vs' - s_0) / \zeta$, $s_0 = \zeta \delta$, and using short-hand notation $\vg^\circ$ for the projected gradient, the above can be simplified:
\begin{align}
  \vg &\gets \widehat \nabla \ell(\vP (\vm' + \vepsilon)), \vg^\circ \gets \vP^\top \vg, \quad \vepsilon \sim \gauss(\vepsilon \mid 0, \vsigma^2) \\
  \vh &\gets (1 - \rho) \vh + \rho \vg^\circ \vepsilon / \vsigma^2, \\
  \vm' &\gets \vm' - \rho \, \text{diag}(\vh + \delta)^{-1} \left(\vg^\circ + \delta \vm' \right).
\end{align}
 Multiplying the last line from the left with $\vP$, we can bring the update to the original $\vparam$-space:
\begin{align}
  \vm &\gets \vm - \rho \, \vP \text{diag}(\vh + \delta)^{-1} \left(\vg^\circ + \delta \vP^\top \vm \right).
\end{align}
To arrive at the form shown in \Cref{alg:evon}, we apply the method to neural networks with $L$ layers, where the weights $\vparam = \text{cat}(\text{vec}(\vParam_1), \hdots, \text{vec}(\vParam_L))$ are assumed to be a concatenation of matrix parameters $\vParam_l$ flattened into vectors via $\text{vec}(\cdot)$. The inverse of $\text{vec}(\cdot)$ is denoted by $\text{mat}(\cdot)$. We also assume $\vP$ to be a block-diagonal matrix with blocks $\vP_1, \hdots, \vP_L$ such that \revision{$\vP_l \vparam_l = \text{vec}(\vQ_L^l \vParam_l (\vQ_R^l)^\top)$ with $\vparam_l = \text{vec}(\vParam_l)$}, that is, $\vP_l = \vQ_R^l \otimes \vQ_L^l$ has Kronecker-factored structure. 

To simplify notation, following \citep{vyas2024soap}, we write the method for a single matrix parameter and drop the index $l$. Capital letters denote matrix version of the parameter that was vectorized before.
\begin{align}
  \vG &\gets \widehat \nabla \ell(\vM + \vQ_L \vE \vQ_R^\top), \quad \vG^\circ \gets \vQ_L^\top \vG \vQ_R ,\quad E_{i,j} \sim \gauss(0, V_{i,j}) \\
  \vH &\gets (1 - \rho) \vH + \rho \vG^\circ \odot \vE / \vV, \\
  \vM &\gets \vM - \rho  \vQ_L \left[ \left( \vG^\circ + \delta \vQ_L^\top \vM \vQ_R \right) / (\vH + \delta)\right] \vQ_R^\top,
\end{align}
where $V_{i,j} = 1 / \zeta (H_{i,j} + \delta)$ are the entries of $\vV = \text{mat}(\vsigma^2)$. Importantly, the matrix $\vV$ is not to be confused with the posterior covariance in the original space. It is the diagonal variance $\vsigma^2$ in the rotated space, reshaped in matrix form. The posterior distribution in the original space for that layer is written in \Cref{eq:evonpost} in the main text. 
The posterior covariance over all parameters is a block-diagonal matrix \revision{with one such block per layer. Each block $(\vQ_R \otimes \vQ_L)\, \text{diag}(\text{vec}(\vV))\, (\vQ_R \otimes \vQ_L)^\top$ has a Kronecker-factored eigenbasis, but is itself not a Kronecker product in general: it is one only if $\vV = \va \vb^\top$ has rank one, in which case it equals $(\vQ_R \text{diag}(\vb) \vQ_R^\top) \otimes (\vQ_L \text{diag}(\va) \vQ_L^\top)$. Since the eigenvalues $\vV$ are unrestricted, SOAP-Bubbles strictly contain the Kronecker-factored covariances with the same eigenbasis.}

To arrive at \revision{EVON (\Cref{alg:evon}), we introduce a few more tricks}. We use a different learning rate \revision{$\beta_2$} for the Hessian and introduce a gradient momentum $\bar \vG$\revision{, which leads} to the following updates:
\begin{align}
  \vG &\gets \widehat \nabla \ell(\vM + \vQ_L \vE \vQ_R^\top), \quad \vG^\circ \gets \vQ_L^\top \vG \vQ_R ,\quad E_{i,j} \sim \gauss(0, V_{i,j}) \\
  \widehat \vH &\gets \vG^\circ \odot \vE / \vV, \\
  \bar \vG &\gets \beta_1 \bar \vG + (1 - \beta_1) \vG^\circ, \\
  \vH &\gets \beta_2 \vH + (1 - \beta_2) \widehat \vH + \half (1 - \beta_2)^2 \revision{(\widehat \vH - \vH)^{\odot 2}} / (\vH + \delta) , \\
  \vM &\gets \vM - \alpha \vQ_L \left[ \left(\bar \vG + \delta \vQ_L^\top \vM \vQ_R \right) / (\vH + \delta) \right]\vQ_R^\top.
\end{align}
In the Hessian update, we also added the quadratic correction term used in IVON~\citep{shen2024variational}. 
These updates give lines 1--8 in \Cref{alg:evon}, where we also added clipping to stabilize the algorithm in practice.  Lines 9, 10, 11 are identical to SOAP and update the matrix $\vP$ during training.  
\section{Theoretical Guarantees on Full Posterior Recovery}
\label{app:postapprox}
In this section, we prove \Cref{thm:exactrecovery} and \Cref{corr:exactrecoverycorr} from the main paper. In order to do so, we first clarify some notation, we give analytical expressions for some quantities like the Hessian and the matrices $\mathbf{L}$ and $\mathbf{R}$, and we provide three technical Lemmas.

\subsection{Notation and Preliminaries}
We consider a linear model $f : \bbR^{d\times o} \times \bbR^d\rightarrow\bbR^o$ with parameter $\vParam\in \bbR^{d\times o}$ and data $\vx\in\bbR^d$
\begin{equation}
    f(\vParam, \vx)=\vParam^\top \vx 
\end{equation}
for which the Jacobian is
\begin{equation}
    \nabla_{\vparam} f(\vParam, \vx)
    =
    \vI_o \otimes \vx^\top 
    \in\bbR^{o\times do}
\end{equation}
where $\vI_o$ is the identity matrix of size $o$ and we keep the notation from the main paper of $\vparam = \text{vec}(\vParam)$. The number of parameters is equal to the product of data and label dimensionality $P=do$.

Given some cost function $c(y, f(\vParam, \vx))$ that measures the offset between label $\vy$ and model prediction $f(\vParam, \vx)$, the loss evaluated over a dataset $\mathcal{D}=\{(\vx_i,y_i)\}_{i=1,\ldots,n}$ is $\ell(\vparam) = \sum_{i=1}^n c(y_i,f(\vParam, \vx_i))$. Unlike the main text, where $\ell$ is the average loss, in this section $\ell$ denotes the summed loss, i.e., $n$ times the loss in \Cref{eq:varlearn}. The variational objective then reads $(\zeta/n)\, \mathbb{E}_{q}[\ell(\vparam)] + \myKL{q(\vparam)}{p(\vparam)}$, so $\zeta$ is replaced by $\zeta/n$ in the stationarity condition below. Positive scalar factors change neither the eigenvectors of $\mathbf{L}$, $\mathbf{R}$ and the expected Hessian, nor whether matrices are simultaneously diagonalizable, so the results below hold under either convention. By chain rule, its gradient is
\begin{align}
    \nabla_{\vparam} \ell(\vparam)
    & =
    \sum_{i=1}^n
    \nabla_{\vparam} f(\vParam, \vx_i) ^\top
    \cdot
    \nabla_{f_i} c(y_i, f_i)
    \\
    & =
    \sum_{i=1}^n
    (\vI_o \otimes \vx_i) 
    \cdot 
    \nabla_{f_i} c(y_i, f_i)
    =
    \sum_{i=1}^n
    \text{vec}
    \left(
        \vx_i
        \nabla_{f_i} c(y_i, f_i)^\top
    \right)
    \in\bbR^{do}
\end{align}

Thus, the gradient in matrix form is a sum of outer products
\begin{equation}
\label{eq:linear_model_gradient_matrix_form}
    \nabla_{\text{\vParam}} \ell(\vparam)
    =
    \sum_{i=1}^n
    \vx_i
    \nabla_{f_i} c(y_i,f_i)^\top
    \quad
    \in\bbR^{d\times o}
\end{equation}
where $f_i$ is a shorthand for $f(\vParam,\vx_i)$.

The expected Hessian under a posterior distribution $q(\vparam)$ is equal to
\begin{align}
    \mathbb{E}_{\vparam\sim q}[\nabla^2_{\vparam} \ell(\vparam)]
    & =
    \sum_{i=1}^n
    \nabla_{\vparam} f(\vParam, \vx_i)^\top
    \mathbb{E}_{\vparam\sim q}[\nabla^2_{f_i} c(y_i,f_i)]
    \nabla_{\vparam} f(\vParam, \vx_i) \\
    & =
    \sum_{i=1}^n
    (\vI_o \otimes \vx_i)
    \mathbb{E}_{\vparam\sim q}[\nabla^2_{f_i} c(y_i,f_i)]
    (\vI_o \otimes \vx_i^\top) \\
    & =
    \sum_{i=1}^n
    \mathbb{E}_{\vparam\sim q}[\nabla^2_{f_i} c(y_i,f_i)]
    \otimes 
    (\vx_i \vx_i^\top)
    \quad \in\bbR^{od\times od}
    \label{eq:linear_model_expected_hessian}
\end{align}
that is a sum of Kronecker products.

The matrices $\mathbf{L}$ and $\mathbf{R}$ are the Exponential Moving Averages (EMA) of outer products of the gradient in matrix form, as in \cref{eq:linear_model_gradient_matrix_form}. 
Assuming to use full-batch gradients, $T$ steps of the optimizer will follow some trajectory $\vparam_0,\ldots,\vparam_T$. The Exponential Moving Average computed iteratively (lines 9 and 10 of \Cref{alg:evon}) is an exponentially weighted sum over the trajectory: with $\mathbf{L}$ and $\mathbf{R}$ initialized at zero, for any function $g$ the following holds: 
\begin{equation}
    \text{EMA}_{\vparam}[g(\vparam)]
    =
    (1-\beta_3)\sum_{t=0}^{T} \beta_3^{T-t}g(\vparam_t).
\end{equation}
The matrix $\mathbf{L}$ is a weighted sum of outer products of data
\begin{align}
    \mathbf{L}
    & =
    \text{EMA}_{\vparam}
    \left[ 
        \nabla_{\text{\vParam}} \ell(\vparam) 
        \cdot \nabla_{\text{\vParam}} \ell(\vparam)^\top 
    \right] \\
    & =
    \text{EMA}_{\vparam}
    \left[ 
        \sum_{i,j=1}^n
        \vx_i
        \cdot
        \nabla_{f_i} c(y_i,f_i)^\top
        \cdot
        \nabla_{f_j} c(y_j,f_j)
        \cdot
        \vx_j^\top
    \right] \\
    & =
    \sum_{i,j=1}^n
    \underbrace{\text{EMA}_{\vparam}
    \left[ 
    \langle
        \nabla_{f_i} c(y_i,f_i),
        \nabla_{f_j} c(y_j,f_j)
    \rangle
    \right]}_{\in\bbR}
    \underbrace{ \vx_i \vx_j^\top}_{\in\bbR^{d\times d}},
    \label{eq:linear_model_expression_of_L}
\end{align}
where the scalar coefficients that weight the sum are given by the EMA of the scalar product between residuals. Similarly, the matrix $\mathbf{R}$ is a weighted sum of outer products of residuals
\begin{align}
    \mathbf{R}
    & =
    \text{EMA}_{\vparam}
    \left[ 
        \nabla_{\text{\vParam}} \ell(\vparam) ^\top 
        \cdot \nabla_{\text{\vParam}} \ell(\vparam)
    \right] \\
    & =
    \text{EMA}_{\vparam}
    \left[ 
        \sum_{i,j=1}^n
        \nabla_{f_i} c(y_i,f_i)
        \cdot
        \vx_i^\top
        \cdot
        \vx_j
        \cdot
        \nabla_{f_j} c(y_j,f_j)^\top
    \right] \\
    & =
    \sum_{i,j=1}^n
    \underbrace{\langle 
        \vx_i,
        \vx_j
    \rangle}_{\in\bbR}
    \underbrace{\text{EMA}_{\vparam}
    \left[ 
        \nabla_{f_i} c(y_i, f_i)
        \cdot
        \nabla_{f_j} c(y_j, f_j)^\top
    \right]}_{\in\bbR^{o\times o}},
    \label{eq:linear_model_expression_of_R}
\end{align}
where the scalar coefficients that weight the sum are given by the scalar product between data. 

Neglecting the QR approximation for simplicity, the rotated eigenbase is defined as 
\begin{equation}
\label{eq:linear_model_P_definition}
    \vP = \text{Eig}(\mathbf{R}) \otimes \text{Eig}(\mathbf{L}) \in \bbR^{od\times od}
\end{equation}
And EVON's variational family (SOAP-Bubbles) is
\begin{equation}
\mathcal{Q}_{\text{bubble}} = \left\{ \, \gauss(\vparam \mid \vm, \vP \text{diag}(\vsigma^2) \vP^\top) \mid \vm \in \mathbb{R}^P, \vsigma \in \mathbb{R}_+^P \, \right\}.
\end{equation}
Is this flexible enough? We study under which conditions this variational family does contain the optima among multivariate Gaussians. To see this, we first need to prove a few technical Lemmas.

\subsection{Technical Lemmas}
\begin{lemma}
\label{lemma:lemma1}
    Let $q\sim\gauss(\mu,\Sigma)$ be a local minima of the variational problem in \cref{eq:varlearn} among the family of multivariate Gaussians
    , then 
    \begin{equation}
    \label{eq:linear_model_lemma_1}
        q\in \mathcal{Q}_{\text{bubble}}
        \qquad
        \Longleftrightarrow
        \qquad
        \vP^\top \mathbb{E}_{\vparam\sim q}[\nabla^2_{\vparam} \ell(\vparam)] \vP  \text{ is diagonal}
    \end{equation}
\end{lemma}
\begin{proof}
    First, given that $q$ is a local minimum, it must satisfy $\Sigma^{-1} = (\zeta/n)\mathbb{E}_{\vparam\sim q}[\nabla^2_{\vparam} \ell(\vparam)] + \Sigma_0^{-1}$. 
    The mode of the Gaussians can always be equated. Then $q\in \mathcal{Q}_{\text{bubble}}$ if and only if its covariance can be expressed by the variational family, i.e.\@ if and only if there exists $\vsigma \in \mathbb{R}_+^P$ such that $\Sigma=\vP \text{diag}(\vsigma^2) \vP^\top$. This, in turn, is true if and only if $ \Sigma^{-1} = \vP\text{diag}(\vsigma^{-2}) \vP^\top $, and such $\vsigma$ exists as long as $\vP^\top\Sigma^{-1}\vP$ is diagonal.
    An isotropic prior $\Sigma_0^{-1}=\alpha\vI$ remains isotropic under any change of basis $\vP^\top\Sigma_0^{-1}\vP=\alpha\vI$. Consequently, $\vP^\top\Sigma^{-1}\vP = \vP^\top\left((\zeta/n)\mathbb{E}_{\vparam\sim q}[\nabla^2_{\vparam} \ell(\vparam)] + \Sigma_0^{-1}\right)\vP = (\zeta/n)\vP^\top \mathbb{E}_{\vparam\sim q}[\nabla^2_{\vparam} \ell(\vparam)] \vP + \alpha\vI$ is diagonal if and only if $ \vP^\top\mathbb{E}_{\vparam\sim q}[\nabla^2_{\vparam} \ell(\vparam)]\vP$ is diagonal.
\end{proof}

In \cref{eq:linear_model_P_definition} multiple eigendecompositions may exist. For example, for the identity matrix, any basis is a valid set of eigenvectors. Once one is fixed, the condition in the RHS of \cref{eq:linear_model_lemma_1} may not hold. This is quite restrictive, and in fact it makes more sense to ask whether \emph{there exists} an eigendecomposition such that the matrix in the RHS of \cref{eq:linear_model_lemma_1} is diagonal. This existence condition is studied in the next lemma, where $\vM$ plays the role of the expected hessian.

\begin{lemma}
\label{lemma:lemma2}
    Let $\vM$ be a symmetric matrix. Then the following holds 
    \[
    \begin{array}{c}
         \text{There exist an eigendecomposition $\text{Eig}(\mathbf{R})$}  \\
         \text{and an eigendecomposition $\text{Eig}(\mathbf{L})$} \\
         \text{such that $\vP^\top \vM \vP$ is diagonal}
    \end{array}
    \,\Longleftrightarrow\,
    \begin{array}{c}
        \text{$\mathbf{R} \otimes \mathbf{L}$ and $\vM$ are simultaneously} \\
        \text{diagonalizable with $\vU=\vU_1\otimes \vU_2$} \\
        \text{that is both $\vU^\top \vM \vU$ and $\vU^\top (\mathbf{R} \otimes \mathbf{L}) \vU$} \\
        \text{are diagonal}
    \end{array}
    \]
\end{lemma}
\textit{Proof idea:} By defining $\vP=\text{Eig}(\mathbf{R})\otimes\text{Eig}(\mathbf{L})$ we are choosing $\vP$ to be an eigenbasis of $\mathbf{R} \otimes \mathbf{L}$. An orthogonal matrix ($\vP$ in this case) diagonalizes a matrix $\vM$ if it is made by an $\vM$-eigenbasis. Then $\mathbf{R} \otimes \mathbf{L}$ and $\vM$ have the same eigenbasis $\vP$.
\begin{proof}
    First recall some known linear algebra properties: an eigendecomposition is an orthogonal matrix that diagonalizes a matrix, and product of orthogonal matrices is orthogonal. The Kronecker product of orthogonal matrices is orthogonal.\\
    \boxed{\Rightarrow} 
    Let $\text{Eig}(\mathbf{R})$ and $\text{Eig}(\mathbf{L})$ be the eigendecompositions in the LHS of the statement, then we have
    \begin{equation}
        \mathbf{R}=\text{Eig}(\mathbf{R}) \vD_R \text{Eig}(\mathbf{R})^\top
        \qquad
        \mathbf{L}=\text{Eig}(\mathbf{L}) \vD_L \text{Eig}(\mathbf{L})^\top
    \end{equation}
    for two diagonal matrices $\vD_R$ and $\vD_L$. Consequently, we have
    \begin{align}
        \mathbf{R} \otimes \mathbf{L} 
        & =
        (\text{Eig}(\mathbf{R}) \vD_R \text{Eig}(\mathbf{R})^\top) \otimes  (\text{Eig}(\mathbf{L}) \vD_L \text{Eig}(\mathbf{L})^\top) \\
        & =
        \underbrace{
            (\text{Eig}(\mathbf{R})\otimes\text{Eig}(\mathbf{L})) 
        }_{=\vP}
        (\vD_R\otimes \vD_L)
        \underbrace{
            (\text{Eig}(\mathbf{R})\otimes\text{Eig}(\mathbf{L}))^\top
        }_{=\vP^\top}
    \end{align}
    where $\vD_R\otimes \vD_L$ is diagonal and $\vP=\text{Eig}(\mathbf{R})\otimes\text{Eig}(\mathbf{L})$ diagonalizes the matrix $\mathbf{R} \otimes \mathbf{L}$. Since $\vP$ diagonalizes also $M$ by hypothesis, $\mathbf{R} \otimes \mathbf{L}$ and $M$ are simultaneously diagonalizable, by definition. 
    
    \boxed{\Leftarrow} Let $\vU=\vU_1\otimes \vU_2$ be the orthogonal matrix that simultaneously diagonalizes $\mathbf{R} \otimes \mathbf{L}$ and $M$. 
    First let's see that both $\vU_1$ and $\vU_2$ must be orthogonal matrices. $\vU \vU^\top=\vI$ implies that
    \begin{equation}
        (\vU_1^\top \vU_1)
        \otimes
        (\vU_2^\top \vU_2)
        =
        (\vU_1^\top\otimes \vU_2^\top)
        (\vU_1\otimes \vU_2)
        =
        (\vU_1\otimes \vU_2)^\top
        (\vU_1\otimes \vU_2)
        =
        \vU^\top \vU
        =
        \vI
        =
        \vI\otimes \vI
    \end{equation}
    and two Kronecker products are equal if and only if both factors are equal up to a scalar constant. This implies that there exists $c\in\bbR_+$ such that 
    $\vU_1^\top \vU_1=c\vI$ and $\vU_2^\top \vU_2=c^{-1}\vI$, so the two matrices are indeed orthogonal, up to a non-zero constant.
    
    The fact that $\vU_1\otimes \vU_2$ diagonalizes $\mathbf{R} \otimes \mathbf{L}$, assumed by hypothesis, means that for some diagonal matrix $\vD$ it holds
    \begin{equation}
        \vD
        =
        (\vU_1\otimes \vU_2)^\top
        (\mathbf{R} \otimes \mathbf{L})
        (\vU_1\otimes \vU_2) 
        =
        (\vU_1^\top \mathbf{R} \vU_1)
        \otimes
        (\vU_2^\top \mathbf{L} \vU_2).
    \end{equation}
    A Kronecker product is diagonal only if both its factors are diagonal, and non-zero. Which implies that there exists two diagonal matrices $\vD_1$ and $\vD_2$ such that
    \begin{equation}
        \vD_1 = \vU_1^\top \mathbf{R} \vU_1
        \qquad
        \vD_2 = \vU_2^\top \mathbf{L} \vU_2
        \qquad
        \vD = \vD_1\otimes \vD_2
    \end{equation}
    Then $c^{-\frac12}\vU_1$ is an eigenbasis for $\mathbf{R}$ and similarly $c^{\frac12}\vU_2$ is an eigenbasis for $\mathbf{L}$. Then
    \begin{equation}
        \vP
        =
        \text{Eig}(\mathbf{R})\otimes\text{Eig}(\mathbf{L})
        =
        (c^{-\frac12} \vU_1) \otimes (c^{\frac12}\vU_2)
        =
        \vU
    \end{equation}
    so $\vP=\vU$ and since $\vU$ diagonalises $\vM$ by hypothesis, also $\vP$ diagonalises $\vM$. Which is equivalent to $\vP^\top \vM \vP$ being diagonal, which concludes the proof.
\end{proof}

We then consider the RHS condition of the previous Lemma in our setting, that is, when the matrix $\vM$ is the expected Hessian $\mathbb{E}_{\vparam\sim q}[\nabla^2_{\vparam} \ell(\vparam)]$. Then we can use the equivalent Hessian formulation from \cref{eq:linear_model_expected_hessian} to prove the following Lemma. Note that this is a sufficient but not necessary condition.

\begin{lemma}
\label{lemma:lemma3}
    The matrices $\mathbf{R}\otimes \mathbf{L}$ and $\mathbb{E}_{\vparam\sim q}[\nabla^2_{\vparam} \ell(\vparam)]$ are simultaneously diagonalizable if
    \begin{itemize}
        \item[(1)] $\mathbf{L}$, $\vx_1 \vx_1^\top$, $\ldots$, $\vx_n \vx_n^\top$ are simultaneously diagonalizable,
        \item[(2)] $\mathbf{R}$, $\mathbb{E}_{\vparam\sim q}[\nabla^2_{f_1} c(y_1,f_1)]$, $\ldots$, $\mathbb{E}_{\vparam\sim q}[\nabla^2_{f_n} c(y_n,f_n)]$ are simultaneously diagonalizable.
    \end{itemize}
    Moreover, the orthogonal matrix $\vU$ that diagonalizes them is also Kronecker factorized.
\end{lemma}
\begin{proof}
    Notice that $\mathbf{R} \otimes \mathbf{L}$ and $\mathbb{E}_{\vparam\sim q}[\nabla^2_{\vparam} \ell(\vparam)]$ are both symmetric, since they are defined as sum and Kronecker products of symmetric matrices. Consequently, if $\mathbf{R} \otimes \mathbf{L}$ and $\mathbb{E}_{\vparam\sim q}[\nabla^2_{\vparam}\ell (\vparam)]$ are simultaneously diagonalizable if and only if they commute, formally
    \begin{equation}
    \label{eq:lemma3_i_dont_know_how_to_call_this3}
        (\mathbf{R} \otimes \mathbf{L})
        \mathbb{E}_{\vparam\sim q}[\nabla^2_{\vparam} \ell(\vparam)]
        =
        \mathbb{E}_{\vparam\sim q}[\nabla^2_{\vparam} \ell(\vparam)]
        (\mathbf{R} \otimes \mathbf{L}).
    \end{equation}
    Following the expression for the expected Hessian in \cref{eq:linear_model_expected_hessian}, this is equivalent to
    \begin{equation}
    \label{eq:lemma3_i_dont_know_how_to_call_this4}
        (\mathbf{R} \otimes \mathbf{L})
        \left(\sum_{i=1}^n
        \mathbb{E}_{\vparam\sim q}[\nabla^2_{f_i} c(y_i,f_i)]
        \otimes
        (\vx_i \vx_i^\top)
        \right)
        =
        \left(\sum_{i=1}^n
        \mathbb{E}_{\vparam\sim q}[\nabla^2_{f_i} c(y_i,f_i)]
        \otimes
        (\vx_i \vx_i^\top)
        \right)
        (\mathbf{R} \otimes \mathbf{L})
    \end{equation}
    which can be equivalently rearranged as
    \begin{equation}
    \label{eq:lemma3_equivalent_formulation}
        \sum_{i=1}^n
        (\mathbf{R} \mathbb{E}_{\vparam\sim q}[\nabla^2_{f_i} c(y_i,f_i)])
        \otimes 
        (\mathbf{L} \vx_i \vx_i^\top)
        =
        \sum_{i=1}^n
        (\mathbb{E}_{\vparam\sim q}[\nabla^2_{f_i} c(y_i,f_i)] \mathbf{R} )
        \otimes 
        (\vx_i \vx_i^\top \mathbf{L} ).
    \end{equation}
    This equality can be satisfied in various ways, and this is the reason why this Lemma only provides a sufficient condition. 
    One possibility (but not the only one) is that the equality holds for each added $i=1,\ldots,n$ in the sum individually, i.e.\@
    \begin{equation}
    \label{eq:lemma3_i_dont_know_how_to_call_this}
        (\mathbf{R} \mathbb{E}_{\vparam\sim q}[\nabla^2_{f_i} c(y_i,f_i)])
        \otimes 
        (\mathbf{L} \vx_i \vx_i^\top)
        =
        (\mathbb{E}_{\vparam\sim q}[\nabla^2_{f_i} c(y_i,f_i)] \mathbf{R} )
        \otimes 
        (\vx_i \vx_i^\top \mathbf{L} )
        \text{ for all } i=1,\ldots,n
    \end{equation}
    In fact if we sum \cref{eq:lemma3_i_dont_know_how_to_call_this} for all $i=1,\ldots,n$ then we obtain exactly \cref{eq:lemma3_equivalent_formulation}. 

    We can equate each factor in the Kronecker products in \cref{eq:lemma3_i_dont_know_how_to_call_this} and obtain 
    \begin{equation}
    \label{eq:lemma3_i_dont_know_how_to_call_this2}
        \mathbf{R} \mathbb{E}_{\vparam\sim q}[\nabla^2_{f_i} c(y_i,f_i)] 
        =
        \mathbb{E}_{\vparam\sim q}[\nabla^2_{f_i} c(y_i,f_i)] \mathbf{R} 
        \quad
        \text{and}
        \quad
        \mathbf{L} \vx_i \vx_i^\top
        =
        \vx_i \vx_i^\top \mathbf{L}
        \text{ for all } i=1,\ldots,n
    \end{equation}
    Note that \cref{eq:lemma3_i_dont_know_how_to_call_this2} implies \cref{eq:lemma3_i_dont_know_how_to_call_this}, \revision{as needed, but the converse} only holds up to scalar factors.

    \cref{eq:lemma3_i_dont_know_how_to_call_this2} contains exactly hypotheses (1) and (2) of the Lemma. Then as we saw \cref{eq:lemma3_i_dont_know_how_to_call_this2} $\Rightarrow$ \cref{eq:lemma3_i_dont_know_how_to_call_this} $\Rightarrow$ \cref{eq:lemma3_equivalent_formulation} $\Leftrightarrow$ \cref{eq:lemma3_i_dont_know_how_to_call_this4} $\Leftrightarrow$ \cref{eq:lemma3_i_dont_know_how_to_call_this3} $\Leftrightarrow$  $\mathbf{R}\otimes \mathbf{L}$ and $\mathbb{E}_{\vparam\sim q}[\nabla^2_{\vparam} \ell(\vparam)]$ are simultaneously diagonalizable.

    This concludes the first part of the proof. We now show that the orthogonal matrix $\vU$ that diagonalizes them is also Kronecker factorized. By hypothesis (1), there exists an orthogonal matrix $\vU_2$ and $n+1$ diagonal matrices $\{\vD_2^{(i)}\}_{i=0,\ldots,n}$ such that
    \begin{equation}
        \vD_2^{(0)}=\vU_2^\top \mathbf{L} \vU_2
        \qquad
        \vD_2^{(i)}=\vU_2^\top
        \vx_i \vx_i^\top
        \vU_2
        \qquad
        \text{for every }i=1,\ldots,n
    \end{equation}
    Similarly, by hypothesis (2), there exists an orthogonal matrix $\vU_1$ and $n+1$ diagonal matrices $\{\vD_1^{(i)}\}_{i=0,\ldots,n}$ such that
    \begin{equation}
        \vD_1^{(0)}=\vU_1^\top \mathbf{R} \vU_1
        \qquad
        \vD_1^{(i)}=\vU_1^\top
        \mathbb{E}_{\vparam\sim q}[\nabla^2_{f_i} c(y_i,f_i)] 
        \vU_1
        \qquad
        \text{for every }i=1,\ldots,n
    \end{equation}
    We argue that $\vU=\vU_1\otimes \vU_2$ is the matrix that diagonalizes both $\mathbf{R}\otimes \mathbf{L}$ and $\mathbb{E}_{\vparam\sim q}[\nabla^2_{\vparam} \ell(\vparam)]$ and concludes the proof. We show these conditions explicitly, respectively
    \begin{align}
        \vU^\top
        (\mathbf{R}\otimes \mathbf{L})
        \vU
        & =
        (\vU_1\otimes \vU_2)^\top
        (\mathbf{R}\otimes \mathbf{L})
        (\vU_1\otimes \vU_2)
        =
        (\vU_1^\top\mathbf{R} \vU_1) \otimes (\vU_2^\top\mathbf{L} \vU_2) \\
        & =
        \vD_1^{(0)} \otimes \vD_2^{(0)} 
        \label{eq:lemma3_diagonal1}
        \\
        \vU^\top
        \mathbb{E}_{\vparam\sim q}[\nabla^2_{\vparam} \ell(\vparam)]
        \vU
        & =
        (\vU_1\otimes \vU_2)^\top
        \left(\sum_{i=1}^n
        \mathbb{E}_{\vparam\sim q}[\nabla^2_{f_i} c(y_i,f_i)]
        \otimes
        (\vx_i \vx_i^\top)
        \right)
        (\vU_1\otimes \vU_2) \\
        & =
        \sum_{i=1}^n
        (\vU_1\otimes \vU_2)^\top
        \left(
        \mathbb{E}_{\vparam\sim q}[\nabla^2_{f_i} c(y_i,f_i)]
        \otimes
        (\vx_i \vx_i^\top)
        \right)
        (\vU_1\otimes \vU_2) \\
        & =
        \sum_{i=1}^n
        \left(
            \vU_1^\top
            \mathbb{E}_{\vparam\sim q}[\nabla^2_{f_i} c(y_i,f_i)]
            \vU_1
        \right)
        \otimes
        \left(
            \vU_2^\top
            (\vx_i \vx_i^\top)
            \vU_2
        \right) \\
        & =
        \sum_{i=1}^n
        \vD_1^{(i)}
        \otimes
        \vD_2^{(i)}
        \label{eq:lemma3_diagonal2}
    \end{align}
    and both \cref{eq:lemma3_diagonal1} and \cref{eq:lemma3_diagonal2} are diagonal matrices.
    This concludes the proof.
\end{proof}

\subsection{Proof of the Main Results}
With these Lemmas in place, we can finally prove \cref{thm:exactrecovery}. We restate it below for convenience.
\exactrecovery*
\begin{proof}
    By hypothesis, both conditions (1) and (2) of \Cref{lemma:lemma3} hold, thus implying that $\mathbf{R} \otimes \mathbf{L}$ and $\vM=\mathbb{E}_{\vparam\sim q}[\nabla^2_{\vparam} \ell(\vparam)]$ are simultaneously diagonalizable with some $\vU=\vU_1 \otimes \vU_2$, that is, both $\vU^\top \vM \vU$ and $\vU^\top (\mathbf{R} \otimes \mathbf{L}) \vU$ are diagonal. 

    That is exactly the assumption on the RHS of \Cref{lemma:lemma2}, and we can apply it to deduce that there exists an eigendecomposition $\text{Eig}(\mathbf{R})$ and an eigendecomposition $\text{Eig}(\mathbf{L})$ such that the orthogonal matrix $\vP=\text{Eig}(\mathbf{R})\otimes\text{Eig}(\mathbf{L})$, as in \cref{eq:linear_model_P_definition}, is an eigenbasis for the expected loss Hessian, that is, that the matrix $\vP^\top \mathbb{E}_{\vparam\sim q}[\nabla^2_{\vparam} \ell(\vparam)] \vP$ is diagonal. 

    Finally, once established that the matrix $\vP$ can be an eigenbasis, \Cref{lemma:lemma1} implies that the stationary points of the variational problem in \cref{eq:varlearn} are a SOAP-Bubble. Equivalently, EVON's variational family \revision{contains} the optimal multivariate Gaussian posterior. This concludes the proof.
\end{proof}

We now derive some sufficient conditions for each of the requirements of the Theorem, two related to the matrix $\mathbf{R}$ and one related to the matrix $\mathbf{L}$. These allow to prove \Cref{corr:exactrecoverycorr}, which we restate for convenience.

\exactrecoverycorr*
\begin{proof}
    \textbf{Sufficient conditions for $\mathbf{R}$ and $\mathbb{E}_{\vparam\sim q}[\nabla^2_{f_i} c(y_i,f_i)]$}\\
    If we are solving a regression task, it means that $c(y_i,f_i)=\|y_i-f_i\|^2$, or some scalar multiple of it. Thus, its Hessian $\nabla^2_{f_i} c(y_i,f_i)$ and consequently the expectation is equal to the identity, or some scalar multiple of it. The hypothesis of \cref{thm:exactrecovery} then follows since any matrix $\mathbf{R}$ is simultaneously diagonalizable with the identity.

    Alternatively, if we are solving a binary logistic regression task, it means that the output dimensionality $o=1$. Consequently, the expected Hessian matrix is simply a scalar and trivially any (scalar) matrix $\mathbf{R}$ is simultaneously diagonalizable with any scalar.
    
    \textbf{Sufficient conditions for $\mathbf{L}$ and $x_i x_i^\top$}\\
    If the optimizer batch size is fixed to 1, then the gradients used to accumulate $\mathbf{L}$ are computed on a single datapoint. Let $i_t$ denote the datapoint used at step $t$, which can be selected in any order. Then the gradient at step $t$ is $\nabla_{\vparam} c(y_{i_t},f(\vparam_t,\vx_{i_t})) = \vx_{i_t} \cdot \nabla_{f_{i_t}} c(y_{i_t},f_{i_t})^\top$, where $f_{i_t} = f(\vparam_t,x_{i_t})$, evaluated at an arbitrary parameter $\vparam_t$. Consequently, the expression of $\mathbf{L}$ in \cref{eq:linear_model_gradient_matrix_form} reduces to
    \begin{align}
        \mathbf{L}
        & =
        \text{EMA}_{t}
        \left[ 
            \nabla_{\text{$\vParam$}} c(y_{i_t},f(\vparam_t,\vx_{i_t}))
            \cdot \nabla_{\text{$\vParam$}} c(y_{i_t},f(\vparam_t,\vx_{i_t}))^\top 
        \right] \\
        & =
        \text{EMA}_{t}
        \left[ 
            \vx_{i_t} 
            \cdot
            \nabla_{f_{i_t}} c(y_{i_t},f_{i_t})^\top
            \cdot
            \nabla_{f_{i_t}} c(y_{i_t},f_{i_t})
            \cdot
            \vx_{i_t} ^\top
        \right] \\
        & =
        \text{EMA}_{t}
        \left[ 
            \|\nabla_{f_{i_t}} c(y_{i_t},f_{i_t})\|^2
            \vx_{i_t} 
            \vx_{i_t} ^\top
        \right] \\
        & =
        \sum_{j=1}^n
        \underbrace{\alpha_j}_{\in\bbR}
        \underbrace{ \vx_j \vx_j^\top}_{\in\bbR^{d\times d}},
        \label{eq:linear_model_expression_of_L_batch_size_1}
    \end{align}
    for some scalar values $\alpha_i$, and where, similarly as before, $\text{EMA}_{t}[g(t)]=(1-\beta_3)\sum_{t=0}^{T} \beta_3^{T-t}g(t)$. 
    If the datapoints are orthogonal, that is, $\langle x_j,x_k\rangle = 0$ for $j\not=k$, then, using the expression of $\mathbf{L}$ in \cref{eq:linear_model_expression_of_L_batch_size_1}, we have that for every $i=1,\ldots,n$
    \begin{equation}
        \mathbf{L} \vx_i 
        =
        \sum_{j=1}^n
        \alpha_j
        \vx_j \vx_j^\top
        \vx_i 
        =
        \|\vx_i\|^2 \alpha_i
        \vx_i  
    \end{equation}
    and similarly that for every $i=1,\ldots,n$
    \begin{equation}
        \vx_i^\top \mathbf{L} 
        =
        \vx_i^\top \sum_{j=1}^n
        \alpha_j
        \vx_j \vx_j^\top
        =
        \|\vx_i\|^2 \alpha_i
        \vx_i^\top  
    \end{equation}
    which we can use to prove that the condition holds, as
    \begin{equation}
        \mathbf{L} \vx_i \vx_i^\top
        =
        \|\vx_i\|^2 \alpha_i \vx_i \vx_i^\top
        =
        \vx_i \vx_i^\top \mathbf{L} 
        \quad
        \text{ for all } i=1,\ldots,n
    \end{equation}
    \revision{The} same equality holds if some datapoints are parallel, rather than orthogonal, that is, if for any $j\not=k$\revision{,} either $\langle \vx_j, \vx_k\rangle = 0$ or $\vx_j=\vx_k$. Thus the proof also extends to repeated datapoints.
\end{proof}

\begin{figure}[!ht]
    \centering
    \vspace{1ex}
    \begin{tabular}{cc}
        \vspace{0.3cm} \phantomsubcaption\label{fig:hess_hist_plain}
        \textbf{(a)} & \begin{minipage}[c]{0.92\linewidth}
            \includegraphics[width=\linewidth]{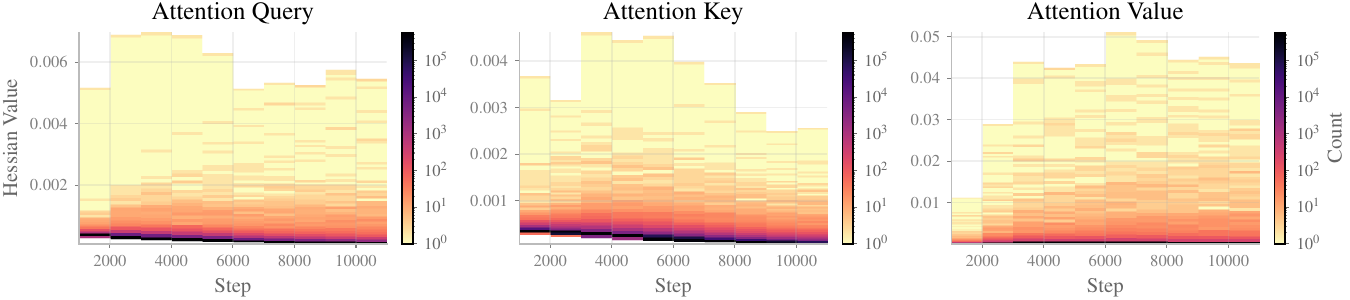}
        \end{minipage} \\ \vspace{0.3cm}
        \phantomsubcaption\label{fig:hess_hist_clip}
        \textbf{(b)} & \begin{minipage}[c]{0.92\linewidth}
            \includegraphics[width=\linewidth]{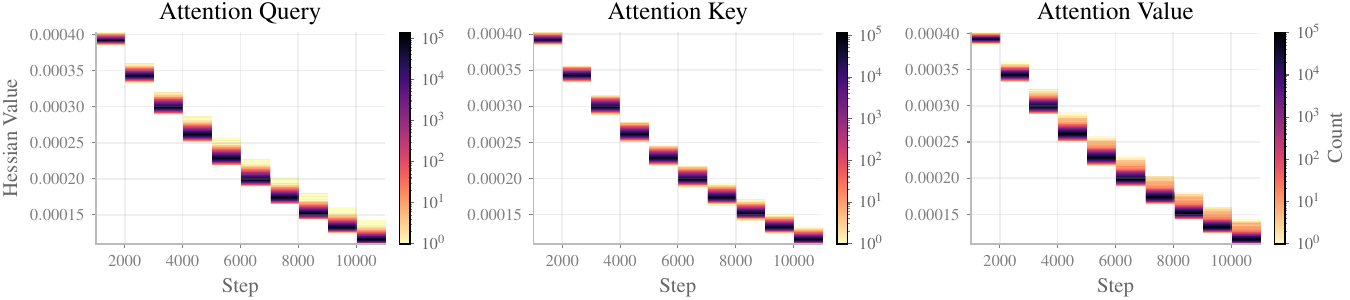}
        \end{minipage} \\ \vspace{0.3cm}
        \phantomsubcaption\label{fig:hess_hist_pg}
        \textbf{(c)} & \begin{minipage}[c]{0.92\linewidth}
            \includegraphics[width=\linewidth]{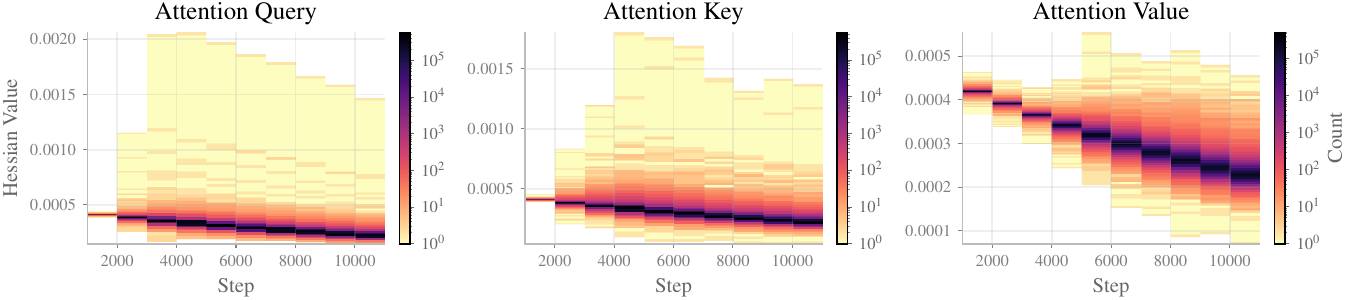}
        \end{minipage} \\ \vspace{0.3cm}
        \phantomsubcaption\label{fig:hess_hist_whitening}
        \textbf{(d)} & \begin{minipage}[c]{0.92\linewidth}
            \includegraphics[width=\linewidth]{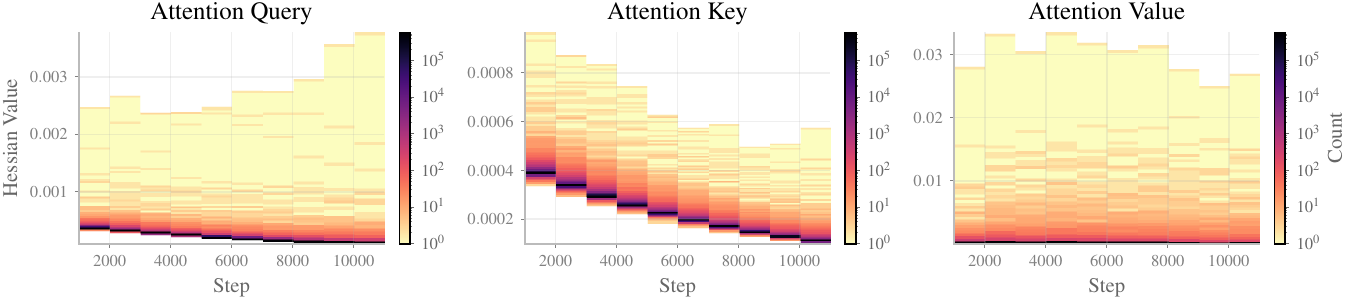}
        \end{minipage}
    \end{tabular}
    \caption{
        \textbf{Evolution of the EVON Hessian distribution in SOAP's eigenspace} (plotted every 1000 training steps) for the attention matrices in the sixth block of NanoGPT.  We compare the baseline (Plain) optimizer (\subref{fig:hess_hist_plain}) as a reference against other variants. Tight Hessian adaptive clipping (HC) (\subref{fig:hess_hist_clip}) results in a much narrower range of smaller Hessian values, which may overestimate the posterior variance. Phased Gradients (PG) (\subref{fig:hess_hist_pg}) also yield a narrower distribution of smaller Hessian values, whereas spectral clipping (SC) (\subref{fig:hess_hist_whitening}) produces a distribution much closer to the baseline's, maintaining similar optimization dynamics.
    }
    \label{fig:hess_hist}
\end{figure}

\section{Investigating Potential Improvements to EVON}
\label{app:variants}

We investigate potential improvements to the EVON optimizer on the NanoGPT pretraining task, reporting our results in \Cref{fig:ablations_eval_curve} and \Cref{tab:ablations}. We fix the hyperparameters of the baseline optimizer and enable or disable the following components: Phased Gradients (PG), Hessian adaptive clipping (HC) with tight ratio (1.5) and spectral clipping (SC) of the update. The table presents the final test loss for each optimizer variant, and the effectiveness of the posterior is studied by performing Bayesian Model Averaging (BMA) with 20 Monte Carlo samples.

\textbf{Spectral clipping (SC).}\label{app:spectral_clipping} Let $\vD = \vQ_L \vU \vQ_R^\top \in \mathbb{R}^{m\times n}$ be the preconditioned update in line 8 of \Cref{alg:evon}, with singular value decomposition $\vD = \vA \vSigma \vB^\top$. Spectral clipping replaces $\vD$ by its orthogonal polar factor $\vA\vB^\top = \vD(\vD^\top\vD)^{-1/2}$, so the mean update becomes $\vM \leftarrow \vM - \alpha\,\vA\vB^\top$. The preconditioner $1/(\vH+\delta)$ still determines the singular directions $\vA$ and $\vB$, but the step length along each of them is set by $\alpha$ alone; the Hessian estimate and the posterior variance $\vV$ (lines 1--6) are unaffected. We approximate the polar factor with the quintic Newton--Schulz iteration of Muon~\citep{jordan2024muon}: starting from $\vX_0 = \vD/\|\vD\|_F$ (transposed if $m>n$), we run five steps of
\begin{equation*}
\vX_{k+1} = a\vX_k + \bigl(b\,\vX_k\vX_k^\top + c\,(\vX_k\vX_k^\top)^2\bigr)\vX_k, \qquad (a,b,c) = (3.4445,\,-4.7750,\,2.0315),
\end{equation*}
in \texttt{bfloat16}. These coefficients trade accuracy for speed: after five steps the singular values lie roughly between $0.5$ and $1.5$ rather than exactly at one. Parameters with more than two dimensions are reshaped into a matrix with the first dimension as rows, and one-dimensional parameters are normalized to unit $\ell_2$ norm. Unlike Muon, we do not rescale the result by $\sqrt{\max(1, m/n)}$. \Cref{alg:evon} also admits element-wise clipping of the update as in IVON; we found spectral clipping to perform considerably better in our experiments, so the released implementation only supports spectral clipping.

\textbf{Phased Gradients (PG).} Phased Gradients alternates every training step between two phases: one phase computes gradients at perturbed parameter values for Hessian estimation, while the other computes gradients at the posterior mean to update the $\mathbf{L}$ and $\mathbf{R}$ preconditioners. Both phases update the posterior mean.

Our findings indicate that Phased Gradients (PG) can accelerate initial convergence when combined with other variants (SC and HC). Phased Gradients and tight Hessian adaptive clipping degrade the overall performance of the optimizer in isolation. Conversely, spectral clipping of the update yields notable improvements in stability and final test loss; the posterior estimation is not significantly affected when we compare the final loss of the posterior mean and BMA (\Cref{tab:ablations}).

To further understand these behaviors, we examine the effects of these modifications on the estimated Hessian. \Cref{fig:hess_hist_plain,fig:hess_hist_clip,fig:hess_hist_pg,fig:hess_hist_whitening} show the histograms of the estimated Hessian in Shampoo's eigenspace for each variant. We observe that both Hessian adaptive clipping (\Cref{fig:hess_hist_clip}) and Phased Gradients (\Cref{fig:hess_hist_pg}) result in narrower distributions with much smaller Hessian values. This behavior may indicate an overestimate of the posterior variance. The spectral clipping variant (\Cref{fig:hess_hist_whitening}) preserves a Hessian distribution and training behavior very close to the baseline's.

\begin{figure}[htbp]
    \centering
    \begin{subfigure}[b]{0.49\textwidth}
        \centering
        \includegraphics[width=\linewidth]{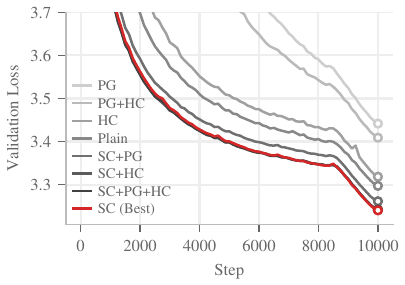}
        \caption{Evaluation loss curves.}
        \label{fig:ablations_eval_curve_sub}
    \end{subfigure}\hfill
    \begin{subfigure}[b]{0.49\textwidth}
        \centering
        \includegraphics[width=\linewidth]{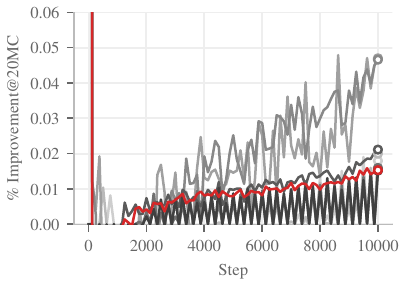}
        \caption{BMA loss improvement.}
        \label{fig:loss_mc_improvement_sub}
    \end{subfigure}
    \caption{
        Evaluation loss curve and BMA test loss improvement for different EVON variants during NanoGPT pretraining. (a) EVON with spectral clipping (SC) achieves the lowest test loss, demonstrating the effectiveness of spectrally clipping the update. In contrast, variants utilizing only Phased Gradients (PG) and Hessian adaptive clipping (HC) underperform both the baseline (Plain) and EVON with spectral clipping. (b) Percentage improvement in test loss when using BMA with 20 Monte Carlo samples compared to the posterior mean.}
    \label{fig:ablations_eval_curve}
\end{figure}

\begin{table}
    \centering
    \begin{minipage}{0.55\textwidth}
        \begin{tabular}{lcc}
\toprule
Config & Loss@mean & Loss@20MC \\
\midrule
SC & 3.241 & 3.240 (\textcolor{TolMutedGreen}{-0.015\%}) \\
SC+PG+HC & 3.242 & 3.241 (\textcolor{TolMutedGreen}{-0.015\%}) \\
SC+HC & 3.243 & 3.242 (\textcolor{TolMutedGreen}{-0.021\%}) \\
SC+PG & 3.262 & 3.261 (\textcolor{TolMutedGreen}{-0.016\%}) \\
Plain & 3.299 & 3.297 (\textcolor{TolMutedGreen}{-0.047\%}) \\
HC & 3.320 & 3.318 (\textcolor{TolMutedGreen}{-0.047\%}) \\
PG+HC & 3.410 & 3.409 (\textcolor{TolMutedGreen}{-0.020\%}) \\
PG & 3.442 & 3.442 (\textcolor{TolMutedGreen}{-0.018\%}) \\
\bottomrule
\end{tabular}

    \end{minipage}\hfill
    \begin{minipage}{0.4\textwidth}
        \caption{ 
            Performance evaluation of EVON variants in the NanoGPT pretraining task. Results include the final test loss, measuring both the loss of the posterior mean, and the loss at 20 Monte Carlo samples after Bayesian Model Averaging (BMA). We highlight the relative change over the mean in \textcolor{TolMutedGreen}{green}. 
        }\label{tab:ablations}
    \end{minipage}
\end{table}
\section{Details on the Experiments and Hyperparameters}
\label{app:hyper}
To ensure reproducibility, we report all hyperparameters for each dataset in 
\Cref{tab:clip_finetuning_sweep_hyperparameters}. EVON as a PyTorch 
optimizer package can be found at \url{https://github.com/team-approx-bayes/evon} and the code to reproduce all figures and tables can be found at \url{https://github.com/team-approx-bayes/evon-experiments}. 

\begin{table}[t!]
\centering
\caption{Hyperparameters for the illustrative examples and logistic-regression experiments. See \url{https://github.com/team-approx-bayes/evon-experiments/tree/master/toyexamples} for additional details.}
\vspace{1ex}
\label{tab:small_scale_hparams}
\small
\setlength{\tabcolsep}{2.6pt}
\renewcommand{\arraystretch}{1.15}
\begin{tabular}{@{}llccccccc@{}}
\toprule
\textbf{Experiment}
& \textbf{Method}
& $\zeta$
& $\alpha$
& $h_0$
& $\delta$
& $\beta_1$
& $\beta_2$
& clip-radius \\
\midrule

\multicolumn{9}{@{}l}{\textit{Neural network regression, \Cref{fig:mlp}(a)}} \\
Small MLP & EVON & $3{\times}10^4$ & $10^{-2}$ & $5{\times}10^{-2}$ & $5{\times}10^{-7}$ & $0.9$ & $0.999$ & -- \\
Small MLP & IVON & $10^4$ & $10$ & $5{\times}10^{-2}$ & $10^{-5}$ & $0.9$ & $0.999$ & $10^{-1}$ \\
Large MLP & EVON & $10^4$ & $10^{-1}$ & $5{\times}10^{-2}$ & $10^{-4}$ & $0.9$ & $0.999$ & -- \\
Large MLP & IVON & $10^5$ & $1$ & $10^{-2}$ & $10^{-5}$ & $0.9$ & $0.9995$ & $10^{-1}$ \\
\midrule

\multicolumn{9}{@{}l}{\textit{UCI Iris multiclass logistic regression, \Cref{fig:mlp}(b)}} \\
LogReg & EVON & $n$ & $5{\times}10^{-2}$ & $10^{-2}$ & $10^{-4}/n$ & $0.9$ & $0.99985$ & -- \\
LogReg & IVON & $n$ & $1$ & $5{\times}10^{-2}$ & $10^{-4}/n$ & $0.9$ & $0.9995$ & $10^{-1}$ \\
\midrule

\multicolumn{9}{@{}l}{\textit{USPS $3$vs$5$ binary logistic regression, \Cref{fig:binary-logreg}}} \\
LogReg & EVON & $n$ & $0.03\to0.001$ & $\frac{1-10^{-5}}{n}$ & $10^{-5}/n$ & $0.9$ & $0.99\to0.999$ & -- \\
LogReg & IVON & $n$ & $5{\times}10^{-3}$ & $\frac{1-10^{-5}}{n}$ & $10^{-5}/n$ & $0.9$ & $0.995\to0.9995$ & $10^{-1}$ \\

\bottomrule
\end{tabular}
\end{table}

\subsection{Illustrative Examples}
\label{app:toy_hyper}
\paragraph{Illustrative experiment in \Cref{fig:fig1}.} The experiment uses a linearly separable binary classification dataset with 2D features, from the Bayesian logistic regression chapter of \citet{Mu12}. Both EVON and IVON are run with small learning rate for many steps to ensure convergence to their respective solution, which are shown as the ellipses in the figure.

\paragraph{Neural network regression experiments in \Cref{fig:mlp} (a).} We follow the setup from \citet[Figure~5]{fadel2025viking}, where our small MLP has the same 1-8-8-1 architecture with $\tanh$-activations. The large MLP has an 1-64-64-64-64-1 architecture with ELU activations. We train using full-batch gradients, as the dataset only consists of 10 data points. IVON uses elementwise clipping, EVON uses whitening with Hessian clipping. The hyperparameters are in \Cref{tab:small_scale_hparams}.

\subsection{Bayesian Logistic Regression}
\label{app:logreg_hyper}
\paragraph{Experiment in \Cref{fig:binary-logreg}.} We follow the setup of \citet[Fig.~1]{mishkin2018slang}, 
which is binary logistic regression on the USPS dataset, considering the 3vs5 classification case. 
The hyperparameters are in \Cref{tab:small_scale_hparams}.

\paragraph{Experiment in \Cref{fig:mlp} (b).} The experiment is multiclass logistic regression on the UCI Iris dataset. The hyperparameters are in \Cref{tab:small_scale_hparams}. For EVON, the eigenbasis is updated every $25$ steps and no Hessian clipping is used. We use Stan \citep{carpenter2017stan} via the CmdStanPy package to compute the reference posterior marginals via Hamiltonian Monte Carlo (HMC) and the No-U-Turn Sampler (NUTS)~\citep{hoffman2014no}.  

\begin{table}[t!]
\centering
\caption{Hyperparameters for NanoGPT and Llama.
EVON uses spectral clipping (\Cref{app:spectral_clipping}) ($\gamma$ is the adaptive Hessian clipping parameter). 
IVON uses element-wise clipping with the listed radius.} \vspace{1ex}
\label{tab:evon_ivon_hparams}
\resizebox{\columnwidth}{!}{%
\begin{tabular}{lcccccccc}
\toprule
\textbf{Setup} 
& ${\zeta}$
& ${h_0}$
& ${\alpha}$
& ${\beta_1}$
& ${1-\beta_2}$
& ${\delta}$
& ${\beta_3}$
& {clipping} \\
\midrule
EVON (NanoGPT)
& $5.27{\times}10^{6}$
& $5.82{\times}10^{-1}$
& $1.72{\times}10^{-2}$
& $0.943$
& $9.12{\times}10^{-4}$
& $1{\times}10^{-6}$
& $0.995$
& $\gamma=1.5$ \\
IVON (NanoGPT)
& $4.66{\times}10^{7}$
& $1.80{\times}10^{-2}$
& $1.42{\times}10^{2}$
& $0.890$
& $1.23{\times}10^{-4}$
& $1{\times}10^{-6}$
& --
&  $10^{-4}$ \\

EVON (Llama)
& $7.61{\times}10^{7}$
& $9.31{\times}10^{-1}$
& $1.52{\times}10^{-2}$
& $0.881$
& $3.46{\times}10^{-4}$
& $8.83{\times}10^{-7}$
& $0.995$
& $\gamma=1.5$ \\

IVON (Llama)
& $5.06{\times}10^{8}$
& $1.95{\times}10^{-3}$
& $1.18{\times}10^{3}$
& $0.826$
& $7.29{\times}10^{-5}$
& $3.93{\times}10^{-7}$
& --
&  $2.49{\times}10^{-4}$ \\
\bottomrule
\end{tabular}%
}
\end{table}

\subsection{Language Model Pretraining}
\label{app:pretrain_hyper}
The experimental setup for NanoGPT~\citep{modded_nanogpt_2024} and Llama~\citep{glentis2025minimal} follows the experiments shown in \citet[Fig.~1]{lin2026understanding}, who provide tuned baselines and hyperparameters for AdamW and SOAP. NanoGPT experiments are run on a single B200 GPU and Llama experiments on a single H100 GPU, all with batch size 512. For NanoGPT, a linear warmup, constant step size and linear cooldown is used. Llama has linear warmup with cosine decay. Both SOAP and EVON update the preconditioner every 10 steps using QR decomposition. For both experiments, EVON and IVON's hyperparameters were tuned using equal budget of 500 runs in a random search over all hyperparameters. All optimizers use AdamW for the embedding layer and 1D tensors, where AdamW uses the same hyperparameters as \citet{lin2026understanding}. The best hyperparameters (wrt final validation loss) found by the random search are summarized in \Cref{tab:evon_ivon_hparams}.

\begin{figure}[t!]
    \centering
    \begin{subfigure}[b]{0.33\textwidth}
        \centering
        \includegraphics[width=\linewidth]{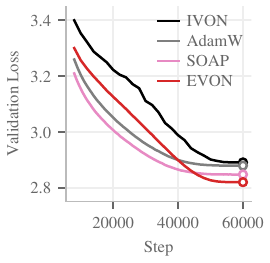}
        \caption{Validation loss.}
        \label{fig:llama-350m-loss}
    \end{subfigure}\hfill
    \begin{subfigure}[b]{0.33\textwidth}
        \centering
        \includegraphics[width=\linewidth]{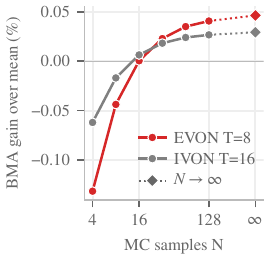}
        \caption{BMA gain over the mean.}
        \label{fig:llama-350m-bma-gain}
    \end{subfigure}\hfill
    \begin{subfigure}[b]{0.33\textwidth}
        \centering
        \includegraphics[width=\linewidth]{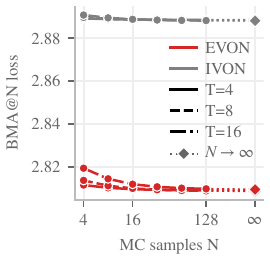}
        \caption{BMA loss for all $T$.}
        \label{fig:llama-350m-bma-loss}
    \end{subfigure}
    \caption{350M LLaMA pretraining on C4. (a) Validation loss: EVON achieves the lowest final validation loss, outperforming SOAP, AdamW, and IVON under the same protocol. (b) Relative improvement of the validation loss of Bayesian Model Averaging (BMA) with $N$ Monte Carlo samples over the posterior mean, and (c) the BMA loss itself, for the final checkpoints. Samples are drawn with the posterior covariance scaled by a temperature $T \in \{4, 8, 16\}$; (c) shows all temperatures, (b) the per-method $T$ with the lowest BMA loss at $N=128$. The $N\to\infty$ limit (diamonds) is estimated by Richardson extrapolation, $2L_{128} - L_{64}$, assuming the BMA loss $L_N$ approaches its limit as $1/N$. As for NanoGPT, averaging only becomes beneficial with enough samples ($N \geq 16$); beyond that, EVON's gain grows faster than IVON's and reaches a larger limit.}
    \label{fig:llama-350m}
\end{figure}

\subsection{CLIP Finetuning}
\label{app:clip_hyper}
We train CLIP models on CARS \citep{cars}, DTD (Describable Textures Dataset) \citep{dtd}, EuroSAT \citep{eurosat}, GTSRB (German Traffic Sign Recognition Benchmark) \citep{gtsrb}, MNIST, RESISC45 \citep{resisc45}, SUN397 \citep{sun397}, and SVHN \citep{svhn}.
For each dataset, we tune the hyperparameters of the optimizers using a random search over the parameter ranges defined in \Cref{tab:clip_finetuning_sweep_hyperparameters}, running 50 iterations for each optimizer and dataset. We then test the best hyperparameters using 5 different random seeds. We set early stopping patience to 3 epochs for all optimizers, monitoring the validation loss. In \Cref{tab:clip_best_hparams_hyperparams} we report the best hyperparameter values found for each optimizer and dataset. 

\subsection{Memory Comparison}
\label{app:memory}
We report peak GPU memory usage for NanoGPT training in \Cref{tab:memory}. EVON and SOAP have similar memory footprints, since both maintain the same Kronecker-factored preconditioner and perform the same QR-based basis updates. EVON's additional cost comes mainly from random-number sampling for posterior perturbation, which does not require additional persistent memory. IVON and AdamW have similar memory footprints, as reported by \citet[App.]{shen2024variational}.

\begin{table}[h]
    \centering
    \begin{minipage}{0.45\textwidth}
        \centering
        \begin{tabular}{lc}
        \toprule
        \textbf{Optimizer} & \textbf{Peak Memory (MiB)} \\
        \midrule
        AdamW & 17{,}011 \\
        SOAP & 19{,}312 \\
        EVON & 19{,}887 \\
        \bottomrule
        \end{tabular}
    \end{minipage}\hfill
    \begin{minipage}{0.5\textwidth}
        \caption{Peak GPU memory usage for NanoGPT training on a single H100 GPU. EVON and SOAP have similar memory footprints. IVON/AdamW use less memory since they do not store $L,R$ kronecker factors and their eigenbases.}
        \label{tab:memory}
    \end{minipage}
\end{table}

\begin{table}[ht]
\centering
\begin{tabular}{lll>{{\raggedright\arraybackslash}}p{6.7cm}}
\toprule
Optimizer & Hyperparameter & Distribution & Range / Values \\
\midrule
 & Learning rate $\alpha$ & Log-uniform & $[10^{-6},\ 5 \times 10^{-4}]$ \\
\multirow{-2}{*}{AdamW} & Weight decay $\lambda$ & Log-uniform & $[10^{-4},\ 10^{-1}]$ \\
\midrule
 & Learning rate $\alpha$ & Log-uniform & $[10^{-4},\ 10^{-1}]$ \\
 & Weight decay $\delta$ & Log-uniform & $[10^{-6},\ 10^{-3}]$ \\
\multirow{-3}{*}{SGD} & Momentum $\mu$ & Uniform & $[0.8,\ 0.99]$ \\
\midrule
 & Learning rate $\alpha$ & Log-uniform & $[10^{-4},\ 5\times10^{-2}]$ \\
 & Weight decay $\delta$ & Log-uniform & $[10^{-7},\ 10^{-2}]$ \\
 & $\zeta$ & Grid & $\{10^7,\ 5{\times}10^7,\ 10^8,\ 5{\times}10^8,\ 10^9,\ 10^{10}\}$ \\
 & Hessian init $h_0$ & Grid & $\{10^{-3},\ 10^{-2},\ 10^{-1}\}$ \\
 & $\beta_1$ & Grid & $\{0.8,\ 0.9,\ 0.95,\ 0.99\}$ \\
\multirow{-6}{*}{IVON} & $1 - \beta_2$ & Grid & $\{10^{-3},\ 5{\times}10^{-4},\ 10^{-4},\ 5{\times}10^{-5},\ 10^{-5}\}$ \\
\midrule
 & Learning rate $\alpha$ & Log-uniform & $[10^{-6},\ 5\times10^{-4}]$ \\
 & Weight decay $\delta$ & Log-uniform & $[10^{-4},\ 2\times10^{-1}]$ \\
 & $(\beta_1,\,1-\beta_2)$ & Grid & $\{0.85,0.9,0.95\}\times\{5{\times}10^{-2},10^{-2}\}$ (6 pairs) \\
\multirow{-4}{*}{SOAP} & $\beta$ for $\vL$ and $\vR$ & Grid & $\{0.9,\ 0.95,\ 0.99\}$ \\
\midrule
 & Learning rate $\alpha$ & Log-uniform & $[10^{-4},\ 10^{-1}]$ \\
 & Weight decay $\delta$ & Log-uniform & $[10^{-7},\ 10^{-2}]$ \\
 & $\zeta$ & Grid & $\{10^7,\ 5{\times}10^7,\ 10^8,\ 5{\times}10^8,\ 10^9,\ 10^{10}\}$ \\
 & Hessian init $h_0$ & Grid & $\{10^{-3},3{\times}10^{-3},5{\times}10^{-3},10^{-2},5{\times}10^{-2},10^{-1}\}$ \\
\multirow{-5}{*}{EVON} & $(\beta_1,\,1-\beta_2)$ & Grid & $\{0.8,0.9,0.95,0.99\}$\newline$\times\{10^{-3},5{\times}10^{-4},10^{-4},10^{-5}\}$\newline$\cup\ \{(0.95,\,5{\times}10^{-2})\}$ (17 pairs) \\
\bottomrule
\end{tabular}

\vspace{1ex}
\caption{\textbf{Hyperparameter sweep ranges for CLIP finetuning.} For continuous hyperparameters we use a log-uniform (learning rate, weight decay) or uniform (momentum) distribution. For discrete hyperparameters we search over the listed grid values.
\label{tab:clip_finetuning_sweep_hyperparameters}}
\end{table}

\begin{table}[ht]
\centering
\begin{tabular}{llcccc}
\toprule
Optimizer & Hyperparameter & CARS & DTD & EuroSAT & GTSRB \\
\midrule
 & $\alpha$ & $3.04\times 10^{-5}$ & $1.7\times 10^{-5}$ & $7.04\times 10^{-6}$ & $9.93\times 10^{-6}$ \\
\multirow{-2}{*}{AdamW} & $\delta$ & $8.91\times 10^{-3}$ & 0.0144 & $2.56\times 10^{-3}$ & $1.19\times 10^{-4}$ \\
\cmidrule{2-6}
 & $\alpha$ & $3.54\times 10^{-4}$ & $1.12\times 10^{-4}$ & $1.43\times 10^{-3}$ & $2.38\times 10^{-3}$ \\
 & $\delta$ & $1.01\times 10^{-3}$ & $3.14\times 10^{-6}$ & $6.89\times 10^{-6}$ & $5.15\times 10^{-4}$ \\
 & $\zeta$ & $5\times 10^{7}$ & $10^{8}$ & $5\times 10^{8}$ & $10^{9}$ \\
 & Hessian init $h_0$ & 0.01 & $10^{-3}$ & 0.1 & 0.01 \\
 & $\beta_1$ & 0.99 & 0.9 & 0.9 & 0.8 \\
\multirow{-6}{*}{IVON} & $1 - \beta_2$ & $10^{-4}$ & $10^{-4}$ & $10^{-4}$ & $10^{-4}$ \\
\cmidrule{2-6}
 & $\alpha$ & $7.43\times 10^{-5}$ & $3.71\times 10^{-5}$ & $9.43\times 10^{-6}$ & $1.56\times 10^{-5}$ \\
 & $\delta$ & 0.04575 & $1.67\times 10^{-4}$ & $4.81\times 10^{-3}$ & $2.49\times 10^{-4}$ \\
 & $(\beta_1, 1-\beta_2)$ & $(0.9,\, 0.01)$ & $(0.85,\, 0.01)$ & $(0.95,\, 0.01)$ & $(0.95,\, 0.01)$ \\
\multirow{-4}{*}{SOAP} & $\beta$ for $\vL$ and $\vR$ & 0.9 & 0.95 & 0.99 & 0.99 \\
\midrule
\rowcolor{gray!20} \cellcolor{white} & \cellcolor{white} $\alpha$ & 0.06263 & $3.74\times 10^{-3}$ & 0.07636 & 0.01957 \\
\rowcolor{gray!20} \cellcolor{white} & \cellcolor{white} $\delta$ & $1.87\times 10^{-7}$ & $1.13\times 10^{-4}$ & $5.98\times 10^{-4}$ & $1.4\times 10^{-6}$ \\
\rowcolor{gray!20} \cellcolor{white} & \cellcolor{white} $\zeta$ & $10^{7}$ & $10^{8}$ & $5\times 10^{7}$ & $10^{7}$ \\
\rowcolor{gray!20} \cellcolor{white} & \cellcolor{white} Hessian init $h_0$ & 0.01 & $5\times 10^{-3}$ & $10^{-3}$ & 0.01 \\
\rowcolor{gray!20} \cellcolor{white} & \cellcolor{white} $\beta_1$ & 0.95 & 0.95 & 0.95 & 0.95 \\
\rowcolor{gray!20} \cellcolor{white} \multirow{-6}{*}{EVON} & \cellcolor{white} $1 - \beta_2$ & 0.05 & 0.05 & 0.05 & 0.05 \\
\midrule
\midrule
Optimizer & Hyperparameter & MNIST & RESISC45 & SUN397 & SVHN \\
\midrule
 & $\alpha$ & $1.98\times 10^{-6}$ & $5.1\times 10^{-6}$ & $1.4\times 10^{-5}$ & $1.46\times 10^{-5}$ \\
\multirow{-2}{*}{AdamW} & $\delta$ & $1.7\times 10^{-4}$ & $1.08\times 10^{-3}$ & $4\times 10^{-4}$ & 0.01245 \\
\cmidrule{2-6}
 & $\alpha$ & $6.75\times 10^{-3}$ & $2.78\times 10^{-3}$ & $5.28\times 10^{-3}$ & $8.09\times 10^{-4}$ \\
 & $\delta$ & $3.39\times 10^{-3}$ & $1.67\times 10^{-3}$ & $3.51\times 10^{-4}$ & $2.88\times 10^{-3}$ \\
 & $\zeta$ & $5\times 10^{8}$ & $5\times 10^{8}$ & $10^{8}$ & $5\times 10^{8}$ \\
 & Hessian init $h_0$ & $10^{-3}$ & 0.01 & $10^{-3}$ & 0.1 \\
 & $\beta_1$ & 0.95 & 0.99 & 0.95 & 0.99 \\
\multirow{-6}{*}{IVON} & $1 - \beta_2$ & $5\times 10^{-4}$ & $10^{-4}$ & $10^{-3}$ & $5\times 10^{-4}$ \\
\cmidrule{2-6}
 & $\alpha$ & $5.73\times 10^{-6}$ & $1.27\times 10^{-5}$ & $7.5\times 10^{-6}$ & $9.2\times 10^{-5}$ \\
 & $\delta$ & $2.57\times 10^{-3}$ & $7.33\times 10^{-3}$ & $3.83\times 10^{-4}$ & $1.89\times 10^{-4}$ \\
 & $(\beta_1, 1-\beta_2)$ & $(0.9,\, 0.05)$ & $(0.85,\, 0.01)$ & $(0.9,\, 0.01)$ & $(0.95,\, 0.05)$ \\
\multirow{-4}{*}{SOAP} & $\beta$ for $\vL$ and $\vR$ & 0.99 & 0.99 & 0.95 & 0.95 \\
\midrule
\rowcolor{gray!20} \cellcolor{white} & \cellcolor{white} $\alpha$ & 0.04798 & 0.09509 & 0.03325 & 0.08859 \\
\rowcolor{gray!20} \cellcolor{white} & \cellcolor{white} $\delta$ & $9.45\times 10^{-6}$ & $1.33\times 10^{-6}$ & $1.08\times 10^{-7}$ & $3.08\times 10^{-6}$ \\
\rowcolor{gray!20} \cellcolor{white} & \cellcolor{white} $\zeta$ & $10^{7}$ & $10^{8}$ & $5\times 10^{8}$ & $10^{10}$ \\
\rowcolor{gray!20} \cellcolor{white} & \cellcolor{white} Hessian init $h_0$ & 0.01 & $10^{-3}$ & $10^{-3}$ & 0.01 \\
\rowcolor{gray!20} \cellcolor{white} & \cellcolor{white} $\beta_1$ & 0.95 & 0.95 & 0.95 & 0.95 \\
\rowcolor{gray!20} \cellcolor{white} \multirow{-6}{*}{EVON} & \cellcolor{white} $1 - \beta_2$ & 0.05 & 0.05 & 0.05 & 0.05 \\
\bottomrule
\end{tabular}

\vspace{1ex}
\caption{\textbf{Best hyperparameters for CLIP finetuning.} For each dataset and optimizer, we report the best hyperparameter values found during the random search. In these experiments, we enable adaptive hessian clipping (with a loose 2.0 ratio) and spectral clipping of the update for EVON.
\label{tab:clip_best_hparams_hyperparams}}
\end{table}

\end{document}